\documentclass[twoside,11pt]{article}
\usepackage{jair, theapa, rawfonts}
\usepackage{soul} 
\usepackage{url}
\usepackage[utf8]{inputenc}
\usepackage[small]{caption}
\usepackage{graphicx}
\usepackage{amsmath}
\usepackage{amsthm}
\usepackage{booktabs}
\usepackage{stmaryrd}
\usepackage{pythonhighlight}

\usepackage{algorithm}
\usepackage[noend]{algcompatible}

\algnewcommand\algorithmicreturn{\textbf{return}}
\algnewcommand\RETURN{\algorithmicreturn}
\algnewcommand\algorithmicprocedure{\textbf{Procedure:}}
\algnewcommand\PROCEDURE{\item[\algorithmicprocedure]}%
\algnewcommand\algorithmicendprocedure{\textbf{end procedure}}
\algnewcommand\ENDPROCEDURE{\item[\algorithmicendprocedure]}%
\algnewcommand{\algvar}[1]{{\text{\ttfamily\detokenize{#1}}}}
\algnewcommand{\algarg}[1]{{\text{\ttfamily\itshape\detokenize{#1}}}}
\algnewcommand{\algproc}[1]{{\text{\ttfamily\detokenize{#1}}}}
\algnewcommand{\algassign}{\leftarrow}

\usepackage{makecell}
\usepackage{tikz}
\usetikzlibrary{shapes}
\usepackage{subcaption}
\usepackage{amssymb}
\usepackage{amsfonts}
\usepackage{cprotect}
\usepackage[T1]{fontenc}
\usepackage{courier}
\usepackage{multirow}

\newtheorem{theorem}{Theorem}
\newtheorem{property}{Property}
\newtheorem{lemma}{Lemma}
\newcommand{\bee}{\textsc{Bee Search}}
\newcommand{\bus}{\textsc{BUS}}
\newcommand{\bustle}{\textsc{Bustle}}
\newcommand{\probe}{\textsc{Probe}}
\newcommand{\heapsearch}{\textsc{Heap Search}}
\newcommand{\brute}{\textsc{Brute}}
\newtheorem{example}{Example}

\jairheading{77}{2023}{1275-1310}{11/2022}{08/2023}
\ShortHeadings{Program Synthesis with Best-First Bottom-Up Search}
{Ameen \& Lelis}
\firstpageno{1275}

\begin{document}

\title{Program Synthesis with Best-First Bottom-Up Search}

\author{\name Saqib Ameen \email saqib.ameen@ualberta.ca \\
       \name Levi H. S. Lelis \email levi.lelis@ualberta.ca \\
       \addr Department of Computing Science, \\
       Alberta Machine Intelligence Institute (Amii),\\
       University of Alberta, Canada
       }

\maketitle

\begin{abstract} 
    Cost-guided bottom-up search (\bus) algorithms use a cost function to guide the search to solve program synthesis tasks. In this paper, we show that current state-of-the-art cost-guided BUS algorithms suffer from a common problem: they can lose useful information given by the model and fail to perform the search in a best-first order according to a cost function. We introduce a novel best-first bottom-up search algorithm, which we call \bee, that does not suffer information loss and is able to perform cost-guided bottom-up synthesis in a best-first manner. Importantly, \bee\ performs best-first search with respect to the \emph{generation} of programs, i.e., it does not even create in memory programs that are more expensive than the solution program. It attains best-first ordering with respect to generation by performing a search in an abstract space of program costs. We also introduce a new cost function that better uses the information provided by an existing cost model. Empirical results on string manipulation and bit-vector tasks show that \bee\ can outperform existing cost-guided BUS approaches when employing more complex domain-specific languages (DSLs); \bee\ and previous approaches perform equally well with simpler DSLs. Furthermore, our new cost function with \bee\ outperforms previous cost functions on string manipulation tasks.
\end{abstract}

\section{Introduction}

Synthesizing computer programs that satisfy a specification is a long-standing problem in Computing Science~\shortcite{1969:PROW,1979:SyntehsisDreamPrograms,1985:AutoSynthesis,1994:LogicProgramSynthesis,2004:ScehmaGuidedSynthesis,2006LezamaSketching,2012:NumberTransfromation} that has received much attention from the Artificial Intelligence~\shortcite{BalogGBNT16,robustfill17,kalyan2018neuralguided,Shin2019SyntheticDF,dreamcoder} and the Programming Language communities~\shortcite{AlbarghouthiGK13,Udupa:2013,Woosuk2018,probe,Ruyi20}. In this paper, we consider 
programming-by-example problems in which a system receives a set of input-output examples and it attempts to synthesize a program that maps each input to the desired output. 

One approach to solving program synthesis tasks is to search for a solution over the space of programs defined by a domain-specific language (DSL). The program space that DSLs induce can be very large, and a considerable amount of research has been devoted to developing more effective search algorithms to solve program synthesis tasks~\shortcite{bustle,probe,Woosuk2018,Alur2017ScalingEP,AlbarghouthiGK13}. Bottom-up search (\bus) is a successful search strategy that starts with the smallest possible DSL programs and iteratively generates larger programs by combining the smaller ones generated by the algorithm~\shortcite{AlbarghouthiGK13,Udupa:2013,sygus}. 

One of the key advantages of \bus\ over other search algorithms such as top-down search approaches~\cite{Woosuk2018,Alur2017ScalingEP} is that \bus\ generates complete programs during the search, which means that the programs can be executed and evaluated. The ability to execute programs allows one to discard observational equivalent programs~\cite{AlbarghouthiGK13}; two programs are observational equivalent if they produce the same output value for a given set of input values. The detection of observational equivalent programs can substantially reduce the number of programs generated during the search. 

However, due to the size of the search space, \bus\ is only able to find solutions to problems that can be solved with short programs. Cost-guided \bus\ algorithms are able to solve more problems than \bus\ because they use a cost function to guide the search toward more promising programs~\cite{tfcoder:2020}. A cost function receives a program and returns a cost value.  
Programs with low-cost values are deemed more promising than programs with high-cost values and are given preference to be used as subprograms of other programs, thus biasing the search. Several systems use cost-guided \bus\ algorithms: \textsc{TF-Coder}~\cite{tfcoder:2020}, \textsc{Probe}~\cite{probe}, \textsc{Bustle}~\cite{bustle}, and \textsc{Heap Search} ~\shortcite{Fijalkow2022ScalingNP}.
%

In this paper, we show that, despite their superior performance to \bus, the guided search algorithms used in \textsc{TF-Coder}, \textsc{Probe}, and \textsc{Bustle} suffer from the same problem: they can lose some of the information given by the cost function because they round off the costs of the programs. As a result, they do not necessarily perform the search in best-first order with respect to the cost of the programs and might evaluate more expensive programs before evaluating cheaper ones. \heapsearch\ is an existing cost-guided \bus\ algorithm that searches in best-first order with respect to a cost function. However, \heapsearch\ searches in a best-first order with respect to the evaluation of the programs. This means that it can generate a very large number of programs that are more expensive than the solution program. 
Moreover, as we show in this paper, \heapsearch\ sacrifices the detection of observational equivalent programs to attain its best-first ordering and it is only able to search in best-first order while using some of the cost functions from the literature.

In our work, we present a taxonomy for the cost functions used in previous work, where we divide them into two families of functions: pre-generation and post-generation. Pre-generation functions are those able to evaluate the cost of a program before the program is even created in memory. Post-generation functions require the program to be in memory so that the program can be executed as part of the computation of its cost. 

In addition to our taxonomy, another contribution of this paper is a best-first bottom-up search algorithm we call \bee, which overcomes the weaknesses of previous cost-guided \bus\ algorithms. \bee\ performs search in a best-first ordering according to cost functions from both the pre-generation and post-generation families of functions. 
Moreover, \bee's best-first search is with respect to the generation of programs. That is, \bee\ does not even create in memory programs that are more expensive than the solution program. Note that other best-first algorithms such as A*~\shortcite{hart1968aFormalBasis} and Dijkstra's algorithm~\cite{dijkstras} can generate states that are more costly than the goal state, which can hurt performance in domains with a large branching factor.  
\bee's generation-time best-first search is achieved by searching in an abstract cost-tuple space. Each state in the cost-tuple space informs which programs should be generated next in search, such that the best-first ordering of programs is attained. Unlike \heapsearch, \bee\ performs observational equivalence checks as regular \bus\ algorithms. 

To highlight \bee's ability to use a wide range of cost functions, we introduce a novel cost function based on the neural network model used in the \bustle\ system. In contrast to \bustle's cost function, our cost function ``relies'' more on the prediction of the neural model and less so on the size of the evaluated programs. 

We hypothesize that \bee\ is able to solve more problems than \probe\ and \bustle\ due to the information these algorithms lose during search. To evaluate our hypothesis, we compare the number of problems \bee\ solves while using the same cost functions \probe\ and \bustle\ used in a set of string manipulation tasks and in a set of bit-vector manipulation tasks. The results show that \bee\ is never worse than \probe\ and \bustle\, and it can solve more problems than them when searching in larger program spaces. We also evaluate \bee's generation-time best-first search by comparing it with \heapsearch\ and with a search algorithm based on the best-first search algorithm used in \brute, an Inductive Logic Programming system~\cite{brute}. Both \heapsearch\ and \brute\ perform best-first search, but not with respect to the generation of programs, as \bee\ does. \bee\ outperforms both \heapsearch\ and \brute\ by a large margin in all the settings evaluated. Finally, the results also show that \bee\ with our novel cost function outperforms all systems tested in the string manipulation domain.

This paper is organized as follows. We start by defining the program synthesis problem~(Section \ref{section:formulation}), then we present existing uninformed and cost-guided bottom-up search algorithms for synthesis and discuss the limitations of a few contemporary cost-guided \bus\ algorithms~(Sections \ref{sec:uninformedbus} and \ref{sec:informedbus}). In Section~\ref{section:cost-functions}, we discuss two cost functions and use them to describe the taxonomy of the cost functions used in the literature, then we present our bottom-up best-first search algorithm \bee{}~(Section \ref{section:bee-search}) and prove the guarantees it provides. In Section~\ref{section:results}, we present empirical results, followed by related work~(Section \ref{section:related-work}) and conclusions~(Section \ref{section:conclusions}). In Appendix
~\ref{appendix:DSL}, we present the DSLs we used.

\section{Problem Formulation}\label{section:formulation}

In program synthesis tasks, one is given a DSL in the form of a context-free grammar $\mathcal{G} = (V, \Sigma, R, I)$. Here, $V$, $\Sigma$, and $R$ are sets of non-terminals, terminals, and relations defining the production rules of the grammar, respectively. $I$ is $\mathcal{G}$'s initial symbol. Figure~\ref{fig:dsl} shows a DSL with $V = \{I\}$, $\Sigma = \{$\texttt{concat}, \texttt{1}, \texttt{2}, $\cdots$, \texttt{1000}$\}$, $R$ represents the production rules (e.g., $I \to$ \texttt{1}); we call non-terminal a production rule whose righthand side contains at least one non-terminal symbol and we call terminal a production rule whose righthand side does not contain a non-terminal symbol. The arity of a non-terminal rule is the number of non-terminal symbols on the rule's righthand side. For example, the arity of rule $I \to \texttt{concat}(I, I)$ is $2$. The arity of terminal rules is $0$. The programs $\mathcal{G}$ accepts determine the programs space. For example, $\mathcal{G}$ accepts \texttt{concat}(\texttt{concat}(\texttt{1}, \texttt{2}), \texttt{3}): $I$ is replaced with \texttt{concat}$(I, I)$; then the leftmost $I$ with \texttt{concat}$(I, I)$ and the rightmost $I$ with \texttt{3}, and so on. The DSL of this example treats all numbers as strings and \texttt{concat}(\texttt{concat}(\texttt{1}, \texttt{2}), \texttt{3}) returns \texttt{123}.

Search algorithms represent programs as abstract syntax trees (ASTs). Figure~\ref{fig:dsl} shows the AST of the program \texttt{concat}(\texttt{concat}(\texttt{1}, \texttt{2}), \texttt{3}). Each node in the AST represents a production rule. Nodes representing a non-terminal rule have a number of children equal to the number of non-terminal symbols in the rule. For example, \texttt{concat} has two children because the rule $I \rightarrow $ \texttt{concat}$(I, I)$ contains two symbols $I$. Nodes representing production rules of terminals are leaves in the AST. Note that each subtree in the AST represents a program. We call the subtrees rooted at a child of node $p$ the subprograms of $p$. For example, \texttt{concat}(\texttt{1}, \texttt{2}) and \texttt{3} are subprograms of the root node of the tree in Figure~\ref{fig:dsl}. 
We say that a program is generated in search when the program's AST is created and stored in memory. We say that a program is evaluated when it is executed. 


\begin{figure}[t!]
    \centering
    \begin{minipage}{0.24\textwidth}
    \begin{align*}
    I \to& \, \texttt{concat}(I, I) | \texttt{ 1 } | \texttt{ 2 } \\ & | \cdots | \texttt{ 1000 } 
    \end{align*}
    \end{minipage}
    \begin{minipage}{0.21\textwidth}
    \begin{tikzpicture}[font=\footnotesize,edge from parent/.style={draw}]
        \tikzstyle{node}=[draw,ellipse,minimum size=0.0cm];
        \tikzstyle{level 1}=[level distance=10mm,sibling distance=20mm]
        \tikzstyle{level 2}=[level distance=10mm,sibling distance=20mm]
        \tikzstyle{level 3}=[level distance=10mm,sibling distance=20mm]
        \tikzstyle{level 4}=[level distance=10mm,sibling distance=20mm]
      \node(0)[node]{\texttt{concat}} 
             child{node[node]{\texttt{concat}}
                child{node[node]{\texttt{1}}}
                child{node[node]{\texttt{2}}}
            }
            child{node[node]{\texttt{3}}};
    \end{tikzpicture}
    \end{minipage}
    \cprotect\caption{DSL and AST for \texttt{concat(concat(1, 2), 3)}, which produces the output \texttt{123}.}
    \label{fig:dsl}
    \end{figure}

In addition to a DSL, a program synthesis task is composed of a set of input values $\mathcal{I}$ and output values $\mathcal{O}$. The task is (i) to derive a program that $\mathcal{G}$ accepts and (ii) to correctly map each of the input values to its corresponding output value. For example, consider a DSL represented with a grammar $\mathcal{G}$ that is identical to the one shown in Figure~\ref{fig:dsl} but augmented with the rules $I \rightarrow \texttt{in}_1 | \texttt{in}_2$, where $\texttt{in}_1$ and $\texttt{in}_2$ are two input values. The program \texttt{concat}($\verb|in|_1, \verb|in|_2$) correctly produces the output value for the following problem: $\mathcal{I} = \{[\text{\texttt{1}, \texttt{2}}], [\text{\texttt{10}, \texttt{10}}]\}$  and $\mathcal{O} = \{[\texttt{12}], [\texttt{1010}]\}$. 



\section{Uninformed Bottom-Up Search (\bus)}
\label{sec:uninformedbus}
\begin{algorithm}[t]
    \caption{Uninformed Bottom-Up Search (\bus)}
    \begin{algorithmic}[1]
    \PROCEDURE \textsc{Uninformed-BUS($\mathcal{G}, (\mathcal{I}, \mathcal{O}$))}
    \REQUIRE $\mathcal{G} = (V, \Sigma, R, I)$, and a set of input-output examples ($\mathcal{I}, \mathcal{O}$).
    \ENSURE Solution program $p$ or $\perp$
    \STATE $s \gets 1$ \label{line:bus_s_i}
    \WHILE{not timeout}  \label{line:bus_main_loop_s}
    \FOR{$p$ in \textsc{Next-Program}$(\mathcal{G}, B, s)$}
    \STATE $o \gets $ \textsc{Execute}$(p, \mathcal{I})$ \label{line:bus_p_exec}
    \IF{$o$ equals $\mathcal{O}$}
    \STATE \textbf{return} $p$ 
    \ENDIF
    \IF{$p$ is not equivalent to any program in $B$} \label{line:ui_bus_eq_st}
    \STATE $B[s]$.add$(p)$     
    \ENDIF\label{line:ui_bus_eq_end}
    \ENDFOR
    \STATE $s \gets s + 1$ \label{line:bus_size_increment}
    \ENDWHILE \\\label{line:bus_main_loop_e}
    \RETURN{} $\perp$ 
    \PROCEDURE \textsc{Next-Program}$(\mathcal{G}, B, s)$
    \REQUIRE $\mathcal{G} = (V, \Sigma, R, I)$, programs bank $B$, and program size $s$.
    \ENSURE Program of size $s$
    \FOR{$r \in R$} \label{line:bus_pgen_start}
    \IF{arity$(r) = 0$ and $\text{size}(r) = s$}  \label{line:bus_p_a0_s}
    \STATE \textbf{yield} $r$ \label{line:bus_p_a0_e}
        \ELSIF{arity > $0$ and size($r$) < $s$} \label{line:bus_p_a>0_s}
    \FOR{$(p_1,\cdots,p_{k})$\textbf{ in }$B\times\cdots \times B$} \emph{\#operation over the values of the dictionary $B$}
    \IF{$\text{size}(r(p_1,\cdots,p_k)) = s$ and $(p_1,\cdots,p_k)$ is type-consistent with $r$}
    \STATE \textbf{yield} $r(p_1, \cdots, p_{k})$ \label{line:bus_pgen_ends}
    \ENDIF
    \ENDFOR
    \ENDIF
    \ENDFOR
    \end{algorithmic}
\label{alg:bus}
\end{algorithm}

\bus\ solves program synthesis tasks by enumerating all programs of size $i$ before enumerating programs of size $i+1$, where size is the number of nodes in the program's AST. \bus\ starts by generating all programs defined by the terminal symbols of the DSL (size 1). Then, it uses the programs of size 1 to generate programs of size 2 through the production rules of the DSL; then it uses the programs of size 1 and 2 to generate programs of size 3, and so on. The search stops when it generates a program that maps the inputs to the outputs or it times out. Instead of size, \bus\ can also be height-based, where the height of the program's AST is considered. Since it has been shown that size-based \bus\ is more effective than height-based \bus\ in the string manipulation domain~\cite{probe}, which we consider in this paper, we only consider the size-based version in our work and call it \bus.

\begin{example}
Consider an example where we need to synthesize a program that produces the output \texttt{100010001000} with the DSL shown in Figure~\ref{fig:dsl} (the input set is empty). The solution to this problem is \texttt{concat}(\texttt{concat}(\texttt{1000}, \texttt{1000}), \texttt{1000}). \bus\ first generates and evaluates all programs of size 1: $\{$\texttt{1}, \texttt{2}, $\cdots$, \texttt{1000}$\}$. Since none of these programs correctly generates the desired output, \bus\ generates the set of programs of size 2, which is empty. Next, \bus\ generates all programs of size 3: $\{$\texttt{concat}(\texttt{1}, \texttt{1}), $\cdots$, \texttt{concat}(\texttt{1000}, \texttt{1000})$\}$. This process stops when the solution is generated while \bus\ produces programs of size 5. 
\label{example:bus}
\end{example} 
The pseudocode for the uninformed bottom-up search is given in Algorithm \ref{alg:bus}. It receives a grammar $\mathcal{G}=(V,\Sigma, R,I)$, and set of input-output examples ($\mathcal{I}, \mathcal{O}$) and returns a program $p$ that is able to map the inputs to the outputs. A failure $\perp$ is returned if no solution program is found. \bus\ starts by initializing the size $s = 1$ (line \ref{line:bus_s_i}), then it enters the main loop, where, in each iteration, it calls \textsc{Next-Program} procedure to generate programs of size $s$. The variable $s$ is incremented by one for the next iteration (line \ref{line:bus_size_increment}). 

\textsc{Next-Program} receives the grammar $\mathcal{G}$, bank of programs $B$, which are indexed by the AST size of the programs, and the size of the target program $s$. \textsc{Next-Program} generates programs of size $s$ using the production rules of the grammar $r \in R$ (lines \ref{line:bus_pgen_start}-\ref{line:bus_pgen_ends}). \textsc{Next-Program} returns the production rule $r$ if it is terminal (lines \ref{line:bus_p_a0_s}-\ref{line:bus_p_a0_e}). Otherwise, if $r$'s arity is greater than $0$, it generates programs with production rule $r$ by taking the Cartesian product of all programs in the bank of programs $B$ such that the following constraints are satisfied: (i) 
$\text{size}(r(p_1,\cdots,p_k) = s$, where function size$(\cdot)$ returns the number of nodes in the program's AST, $k$ represents the arity of the production rule $r$, and (ii) 
$(p_1,\cdots,p_k)$ is type-consistent with rule $r$, i.e., the type of subprograms, $p_1, \cdots, p_k$ matches the type of the arguments required by $r$ (e.g., \texttt{concat} can only take arguments of the string type). 

Once a program $p$ is yielded to the main loop, \textsc{Uninformed-BUS} executes $p$ (line \ref{line:bus_p_exec}) and, if the output $o$ of $p$ matches the desired output $\mathcal{O}$, \textsc{Uninformed-BUS} returns the program $p$ as a solution to the problem. Otherwise, it checks for observational equivalence, i.e., whether the search has previously seen a program with the same output set $o$. If not, the search adds the program $p$ to the bank of programs $B$, indexed by the program size $s$. The search continues while it has not timed out and a solution program is not found.


\section{Guided Bottom-Up Search}\label{sec:informedbus}

\textsc{TF-Coder}~\cite{tfcoder:2020}, \textsc{Probe}~\cite{probe}, \heapsearch{}~\cite{Fijalkow2022ScalingNP}, and \textsc{Bustle}~\cite{bustle} use a cost function $w$ to guide the bottom-up search. The function $w$ these systems employ favors programs that are ``more likely'' to lead to a solution. For example, in the problem described above, a cost function could favor programs that produce outputs with digits \texttt{1} and \texttt{0} as they appear in the desired output. In this section, we explain existing cost functions and then explain \probe, \bustle, \heapsearch{}, and \brute{}, which are used as baselines in our experiments. Since \textsc{TF-Coder}'s cost function requires a manually crafted set of weights for each operation of the language, we did not consider it in our experiments.

\subsection{Cost Functions}\label{section:cost-functions}

We divide the cost functions from the literature into two types: \emph{pre-generation} and \emph{post-generation}. Pre-generation cost functions define the cost of a program $p$ based on the production rule used to generate $p$ and on the subprograms of $p$. For example, considering the DSL in Figure~\ref{fig:dsl}, a pre-generation function would determine the cost of the program \texttt{concat(1, 2)} as a function of the cost of the production rule $I \rightarrow \texttt{concat}(I, I)$ and of the subprograms \texttt{1} and \texttt{2}. The cost functions used in \textsc{TF-Coder}, \probe, and \heapsearch{} are pre-generation functions where the cost of a program $p$ is given by the sum of the cost of the production rule used to generate $p$ and the cost of $p$'s subprograms. We call these functions pre-generation because one can compute the cost of a program before generating the program. Post-generation functions determine the cost of a program $p$ while using information that requires the execution of $p$. The cost function used in \bustle\ is post-generation because it uses the output of $p$ to compute its cost. We call these functions post-generation because the AST of the program must be in memory to compute its cost. 

\subsubsection{\probe{} Cost Function ($w_\probe{}$)}\label{section:cost-function-probe}

\begin{figure}[t!]
    \centering
       \begin{tabular}{lll} 
       & $\mathbb{P}_r$ & Cost \\
       \cmidrule{2-3}
       $I \to \, \texttt{concat}(I, I)$  & $0.00005$ & $14.28771$ \\
		\hspace{0.1in} $| \texttt{ 1 } | \texttt{ 2 } | \cdots | \texttt{ 999 }$ & $0.00099984$ & $9.966013$  \\
\hspace{0.1in} $| \texttt{ 1000 }$ & $0.00110895$ & $9.816589$ \\
       \end{tabular}
  \caption{A probabilistic context-free grammar (PCFG) for string manipulation task with probability of each rule ($\mathbb{P}_r$) and its cost which is negative log of the probability (-$\log(\mathbb{P}_r)$).}
    \label{fig:pcfg}
\end{figure}

\probe{} uses a pre-generation cost function ($w_\probe{}$) based on a probabilistic context-free grammar (PCFG). The PCFG assigns a value to each production rule $r$ denoting the probability that $r$ is part of a solution. Consider the PCFG shown in Figure~\ref{fig:pcfg}, \textsc{Probe} transforms the probability of a rule $r$, denoted by $\mathbb{P}_r$, into cost by taking $-\log_2(\mathbb{P}_r)$. The cost of each rule is shown in the column ``Cost''. The cost of a program $p = r(p_1, \cdots, p_{k})$, denoted by $w(p)$, is given by the sum of the costs of its subprograms and the rule $r$ used to derive it:
\begin{equation}
            w(p) = w(r) + \sum_{i=1}^k{w(p_i)}
\label{eq:program_cost}
\end{equation}

\begin{example}
    Consider program $p = $ \texttt{concat}(\texttt{1}, \texttt{2}). The cost of the program is given as $w(p) = 14.28771 + 9.966013 + 9.966013 = 34.219736$ because the cost of \texttt{concat}, \texttt{1}, and \texttt{2} is -$\log(0.00005) = 14.28771$, -$\log(0.00099) = 9.96601$, and -$\log(0.001108) = 9.81658$ respectively. Similarly, the cost of $p = $ \texttt{concat}(\texttt{1000}, \texttt{1000}) is $w(p) = 14.28771 + 9.816589 + 9.816589 = 33.920888$. Furthermore, \probe{} rounds off the cost of the programs to the nearest integer. For example, the cost of $p = $ \texttt{concat}(\texttt{1}, \texttt{2}) would be $34$ as given by $w_\probe{}$. 
\end{example}

The cost function $w_{\probe}$ rounds off the cost value to the nearest integer because \probe{} enumerates the programs in increasing order of integer $w$ values: first it enumerates all the programs of cost $1$, then the ones with integer $w$-values of $2$ and so on, until a solution is found. 
\probe\ learns the PCFG while searching. 
It runs the search until a budget \textsc{Lim} is exhausted and uses the partial solutions encountered in this search to train the PCFG. A partial solution is a program $p$, which maps at least one input from the input set $\mathcal{I}$ to its corresponding output in the set $\mathcal{O}$. 
The budget \textsc{Lim} is defined as a constant $d$, which is defined manually, multiplied by the highest cost $l$ of a production rule in the current PCFG. 
\begin{equation*}
\centering
\textsc{Lim} = \emph{l} \times d, \;\; \text{where} \;\; \emph{l} = \max_{r \in R}(w(r)) \,.
\end{equation*}
After training the PCFG with partial solutions, the search is restarted with the updated PCFG. The parameter $d$ allows one to define how often the system trains the PCFG model and restarts the search; \citeauthor{probe} used $d=6$.  If the search cannot find a solution after restarting and no partial solution is found, then the budget is increased to $\textsc{Lim}_{i} = \textsc{Lim}_{i-1} + \emph{l} \times d$, where $\textsc{Lim}_{i-1}$ is the budget of the previous iteration. 

\probe's PCFG starts with a uniform probability distribution to all production rules, and it updates the probability distribution with partial solutions (programs) as follows. 
\probe\ selects a subset of partial solutions from all the partial solutions encountered in the current iteration that satisfy the following: (i) it is the first cheapest program according to the current cost model, (ii) satisfies a unique subset of input-output examples, and (iii) was not encountered in previous iterations. Then it updates the probability of all production rules $r \in R$, $\mathbb{P}(r)$ with
\begin{equation*}
\centering
\mathbb{P}(r) = \frac{{\mathbb{P}_u(r)}^{1-\textsc{Fit}}}{Z} \;\; \text{where} \;\; \textsc{Fit} = \mathop{\max_{\{ p \in PSol | r \in tr(p) \}}} \frac{\vert o(p) \cap \mathcal{O}  \vert}{\vert \mathcal{O} \vert}
\end{equation*}

Here, $\mathbb{P}_u(r)$ represents the probability of rules as given by the uniform distribution, $Z$ represents the normalization factor, $PSol$ indicates a subset of partial solutions selected using the aforementioned criteria, $tr(p)$ represents the trace of a program, which is the sequence of production rules used to derive the program $p$, $o(p)$ indicates the output of the partial program $p$, and \textsc{Fit} indicates the highest proportion of input-output examples solved by any partial solution $p \in PSol$ derived using production rule $r$. This way, \probe\ increases the probability of production rules $r$ that solve the maximum number of input-output examples. Once \probe{} updates the PCFG, it stores $PSol$ in memory and maintains it across the restarts to ensure that partial solutions selected in previous iterations are not selected again to update the grammar. Therefore, PCFG is only updated when a set $PSol$ with novel partial solutions is found. 
Due to \textsc{Probe}'s update rule, the probabilities are in the open interval $(0.0, 1.0)$. This way, the model is never ``certain'' that a symbol must or must not be used in the solution of a problem; no symbol in the language costs $0$ and $w$ is monotonically increasing.

\subsubsection{\bustle{} Cost Function ($w_\bustle{}$)}\label{section:bustle-cost-function}

\bustle's cost function, $w_\bustle{}$, uses a neural network to compute the probability that a program is part of a solution. The network is a binary classification model that receives the input-output pairs $(\mathcal{I},\mathcal{O})$ of the task and the output of a program $p$ to each of the input values in $\mathcal{I}$ and returns the probability that $p$ is a subprogram of a solution to the task. 

\bustle's cost function is defined using two functions: $w$ and $w'$. For program $p=r(p_1,\cdots, p_k)$ the function $w(p)$ is defined as 
\begin{equation*}
            w(p) = 1 + \sum_{i=1}^k{w'(p_i)} \,.
\end{equation*}
Here, $1$ is the cost of production rule $w(r)$ used to generate $p$ (\bustle\ assumes all operators to cost $1$), and $w'(p_i)$ is the cost of the subprogram $p_i$ as given by the following equation.
\begin{equation}
   \label{eq:reweight}
   \textit{w}'(p) = \textit{w}(p) + 5 - \delta(p) \quad \text{where} \; \delta(p) \in \{0, \cdots, 5\}
\end{equation} 
The value of $\delta(p)$ is an integer value that is based on the probability of $p$ being a subprogram of a solution that the neural model returns. The integer value $\delta(p)$ returns is defined according to a binning scheme. Consider the values $\{0.0, 0.1, 0.2, 0.3, 0.4, 0.6, 1.0\}$; if the probability that the neural model returns is within the first two values, i.e., $[0.0, 0.1)$, then $\delta(p) = 0$, if it is within the second and third values, then $\delta(p) = 1$, and so on. The value of $\delta$ is used to penalize $p$ by changing its cost according to the probability given by the neural network; lower probabilities will result in higher costs. For example, consider $p$ with probability $0.05$, then $\delta(p) = 0$, and $w'(p) = w(p) + 5$. This will delay the use of the subprogram $p$ to generate further programs. Note that similarly to \probe{}, \textsc{Bustle} uses a discretization scheme to ensure that the cost values $w$ and $w'$ are integers. $w_\bustle$ also increases monotonically.

\begin{example}
Consider the generation of \texttt{concat}(\texttt{concat}(\texttt{1}, \texttt{3}), \texttt{2}). \textsc{Bustle} computes its $w$-value as $1 + w'\text{(\texttt{concat}(\texttt{1}, \texttt{3}))} + w'\texttt{(2)}$, i.e., the cost of $1$ for the operator \texttt{concat} plus the $w'$-costs of its subprograms \texttt{concat}(\texttt{1}, \texttt{3}) and \texttt{2}. Once the program \texttt{concat}(\texttt{concat}(\texttt{1}, \texttt{3}), \texttt{2}) is generated, a neural network is used to calculate its $w'$-value, which is used when it appears as a subprogram in another program. 
\end{example}

\paragraph{Property Signatures}
The binary classification model that defines $w_{\bustle}$ receives as input the set of input-output pairs, which could be of varied length. Instead of training a recurrent model to handle inputs of varied size, $w_{\bustle}$ uses property signatures~\cite{Odena2020Learning} to define the input to a simpler fully connected feed-forward neural model. A property is defined as a function $f$ that takes as input the input-output pair of a program $p$ and returns a Boolean: $f(i, o) \rightarrow \{0, 1\}$. A property is used to define some aspect of $p$. For example, given an input-output pair (\texttt{hello world}, \texttt{hello}), a property function $f(i, o)$ that checks whether $o$ is in $i$ returns True to the input-output pair. Similarly, when a list of input-output pairs of a program $p$ is evaluated with a list of $k$ properties, we get a feature vector of length $k$, where each entry indicates the result of a property for all input-output pairs. Each entry of this vector has a value in $\{-1, 0, 1\}$, where the values of $-1$ and $1$ indicate that the property returned either False or True, respectively, to all input-output pairs; the value of $0$ indicates that the property returned True to some pairs and False to others.

\begin{example}
Consider the set of input-output pairs and three properties, written in Python, shown in Figure~\ref{fig:ps_example}. The first lambda function returns True and the third False to all pairs. The second property returns True to the first and third pairs and False to the second. Thus, this set of input-output pairs have the property signature vector $[1,0,-1]$.
\end{example}

\begin{figure}
    \centering
    \begin{python}
    io_pairs = [("hello world", "hello"), 
                ("FOO BAR", "foo"), 
                ("switch", "switch")]
    ps = [lambda inp, out: out.lower() in inp.lower(), 
          lambda inp, out: out in inp,
          lambda inp, out: len(inp) < len(out)]
    \end{python}
    \caption{A set of input output pairs (\texttt{io\_pairs}) and property signature list (\texttt{ps}). The first property returns True for all pairs; 
    the second property returns True for the first and third pairs and False for the second pair; and the third property returns False for all pairs.
    }
    \label{fig:ps_example}
\end{figure}

The model $w_{\bustle}$ receives an input of size fixed by a number of properties. That way, the number of input-output pairs can vary, but the input size remains the same. 

\subsubsection{Cost Functions for \bee}

In this section, we show how we adapt $w_{\probe}$ and $w_{\bustle}$ to \bee. We also introduce a novel cost function to be used with \bee, which is based on $w_{\bustle}$. 

The difference between \textsc{Probe}'s $w$ and \textsc{Bee Search}'s version of it is that in the latter the costs are not rounded off. We denote both versions of the cost function by $w_{\probe}$. If used in the context of \bee, then we refer to the function that does not truncate the values; we refer to the original function of \probe{} in all other contexts. 


Since \bee\ handles real-valued cost functions, we adapt the $w$ function of \textsc{Bustle} by interpolating the values of $\delta$ from $0$ to $5$ according to \textsc{Bustle}'s binning scheme. We interpolate $x = \{0.00, 0.15, 0.25, 0.35, 0.50, 1.00 \}$ and $y = \{0, 1, 2, 3, 4, 5 \}$ by pairing the values $x$ and $y$: $(0.00, 0), (0.15, 1), (0.25, 2), (0.35, 3), (0.50, 4), (1.00, 5)$ and using a cubic spline to obtain an interpolant. Then, given a probability value returned by the model, the interpolant returns the $\delta$-value used to obtain $w'$, as shown in Equation~\ref{eq:reweight}. 
In the context of \bee, we use $w_{\bustle}$ to refer to the interpolated version of the original $w_{\bustle}$. 

We also introduce a novel cost function based on $w_{\bustle}$ defined as follows. 

\begin{equation}\label{eq:new-cf}
w_{\textsc{u}}(p) = 1+ \sum_{i = 1}^k w'_{\textsc{u}}(p_{i}) \,
\end{equation}
where,
\begin{equation}
w'_{\textsc{u}}(p_{i}) = w_{\textsc{u}}(p_{i}) - \log_2 \mathbb{P} (p_i) \,
\end{equation}

$\mathbb{P}(p_i)$ is the probability that $p_i$ is part of a solution according to \textsc{Bustle}'s neural network. This function computes the cost $w$ of a program as the sum of the costs $w'$ of its subprograms added to $1$, the cost of a production rule; the cost $w_{\textsc{u}}(p)$ for all terminal symbols $p$ is $1$. The cost $w_{\textsc{u}}'(p_{i})$ is given by $w_{\textsc{u}}(p_{i})$ added to the negative of the log of the probability that $p_i$ is part of a solution. \textsc{Bustle}'s cost function limits how much the neural model can change the cost of a program by mapping the probabilities to a value between $0$ and $5$. Our cost function leaves the influence of the neural model unbounded by using $- \log_2 \mathbb{P}(p)$. The subscript $\textsc{u}$ in $w_{\textsc{u}}$ stands for ``unbounded''.


We call $w_{\bustle}$ and $w_{\textsc{u}}$ \emph{penalizing functions} because the post-generation $w'$-value is not smaller than the $w$-value. We call $w_{\probe}$, $w_{\bustle}$, and $w_{\textsc{u}}$ \emph{additive} cost functions because the $w$-value of a program $p$ is computed by adding the cost of the subprograms of $p$. The adapted versions of $w_\probe$ and $w_\bustle$, and $w_\textsc{u}$ are monotonically increasing cost functions.  


In the next section, we describe a generic cost-guided bottom-up search algorithm that can be used to instantiate \probe\ and \bustle\ by changing the cost function the search uses. \heapsearch{} and \brute\ are described in Sections~\ref{sec:heap_search} and \ref{sec:brute_search}, respectively. 

\subsection{Generic Cost-Guided Bottom-Up Search}\label{section:cost-guided-bus}

\begin{algorithm}[t]
    \caption{Generic Cost-Guided Bottom-Up Search}
    \begin{algorithmic}[1]
    \PROCEDURE \textsc{Cost-Guided-BUS($\mathcal{G}, (\mathcal{I}, \mathcal{O}), w$)}
    \REQUIRE $\mathcal{G} = (V, \Sigma, R, I)$, set of input-output examples ($\mathcal{I}, \mathcal{O}$), and a cost function $w$.
    \ENSURE Solution program $p$ or $\perp$
    \STATE $c \gets 1$ \label{line:cg_bus_c=0}
    \WHILE{not timeout} \label{line:cg_bus_main_loop_s}
    \FOR{$p$ in \textsc{Next-Program}($\mathcal{G}, B, c, w$)}
    \STATE $o \gets $ \textsc{Execute}$(p, \mathcal{I})$ \label{line:gbus_execute}
    \IF{$o$ equals $\mathcal{O}$}
    \STATE \textbf{return} $p$ 
    \ENDIF
    \IF{$p$ is not equivalent to any program in $B$} \label{line:bus_eq_st}
    \IF{post-generation cost function} \label{line:reweight-start} 
    \STATE $c \gets w'(p)$ \label{line:gbus_reweight} 
    \ENDIF \label{line:reweight-end}
    \STATE $B[c]$.add$(p)$ \label{line:gbus_store_progs}    
    \ENDIF\label{line:bus_eq_end}
    \ENDFOR
    \STATE $c \gets c + 1$ \label{line:cg_bus_c++}
    \ENDWHILE \label{line:cg_bus_main_loop_e}
    \STATE \textbf{return} $\perp$ 
    \PROCEDURE \textsc{Next-Program($\mathcal{G}, B, c, w$)}
    \REQUIRE $\mathcal{G} = (V, \Sigma, R, I)$, programs bank $B$, cost $c$, and a cost function $w$.
    \ENSURE Program of cost $c$
    \FOR{$r \in R$} \label{line:gbus_pgen_start} 
    \IF{arity = $0$ and $w$($r$) = $c$} \label{line:gbus_pgen_a0_s} 
    \STATE \textbf{yield} $r$ \label{line:gbus_pgen_a0_e}
    \ELSIF{arity > $0$ and $w$($r$) < $c$} \label{line:gbus_pgen_a>0_s}

    \FOR{$(p_1,\cdots,p_{k})$\textbf{ in }$B\times\cdots \times B$} \emph{\#operation over the values of the dictionary $B$}
    \IF{$w(r(p_1, \cdots, p_{k})) = c$ and $(p_1,\cdots,p_k)$ is type-consistent with $r$}
    \STATE \textbf{yield} $r(p_1, \cdots, p_{k})$ \label{line:bus_pgen_ends}
    \ENDIF
    \ENDFOR
    \ENDIF
    \ENDFOR
    \end{algorithmic}
\label{alg:guided-bus}
\end{algorithm}

In cost-guided bottom-up search, programs are enumerated in the order of increasing cost $c$. The cost is assigned by a cost function $w$, such that programs that are more likely to lead to the solution have a lower cost. The search enumerates all programs of cost $c$ before enumerating programs of cost $c+1$. It begins by enumerating all programs of cost $1$, then in the second iteration, it uses the production rules $r \in R$ to combine the programs with cost $1$ and generate programs of cost $2$, and so on. The search continues until the solution program $p$ is found, or the search budget is exhausted and a failure $\perp$ is returned. 


\begin{example}


Consider the DSL shown in Figure \ref{fig:pcfg} along with the cost of each production rule $r \in R$, where the goal is to synthesize the program \texttt{concat(concat(1000, 1000), 1000)}. Costs are assigned in a way that they bias the search toward the solution (the symbol \texttt{1000} is cheaper than \texttt{1}, \texttt{2}, $\cdots$, \texttt{999}). The cost of program $p=r(p_1,\cdots,p_k)$ given by rule $r$ is equal to the sum of the costs of its subprograms $p_1, \cdots, p_k$ and the production rule $r$. Furthermore, the cost of each program is rounded to the nearest integer value; for instance, the cost $9.816589$ for program \texttt{1000} is rounded to $10$. Cost-guided \bus\ will start by enumerating all programs with cost $10$: \{\texttt{1}, \texttt{2}, $\cdots$, \texttt{1000}\}, since it is the cheapest set of programs that can be generated. No program can be generated with costs in $[1, 9]$. Next, it generates programs with cost $34$: \{\texttt{concat(1, 1)}, \texttt{concat(1, 2)}, $\cdots$, \texttt{concat(1000, 1000)}\}. For example, the cost of \texttt{concat(1000, 1000)} is calculated as follows. Since the cost of \texttt{concat} is $14.28771$ and the cost of \texttt{1000} is $9.8165$, which gives $14.28771 + 2 \times 9.8165 = 33.920888 \approx 34$. In the third iteration, it generates programs of cost $58$ and it finds the solution program \texttt{concat(concat(1000, 1000), 1000)}. 
\label{example:generic_cost}
\end{example}

Unlike the size-based enumeration, the cost-guided \bus\ is biased toward cheaper programs according to the cost function $w$, which often allows the search to solve the problem while possibly generating many fewer programs compared to the size-based methods. 

The pseudocode for generic cost-guided \bus\ is given in Algorithm \ref{alg:guided-bus}. It receives a grammar $\mathcal{G}$, a set of input-output examples ($\mathcal{I}$, $\mathcal{O}$), and a cost function $w$ to guide the search. It starts by initializing cost $c = 1$, and in each iteration of the main loop (lines \ref{line:cg_bus_main_loop_s}-\ref{line:cg_bus_main_loop_e}), it calls the procedure \textsc{Next-Program} to generate programs with a target $w$-value of $c$. The cost $c$ increases by one after each iteration (line \ref{line:cg_bus_c++}). \textsc{Next-Program} procedure iterates over all the rules $r \in R$ of grammar $\mathcal{G}$ to generate programs with the target cost $c$. If the arity of the rule $r$ is $0$, i.e., it is a terminal rule, and the cost of the rule $w(r) = c$, then it returns the program given by rule $r$ (lines \ref{line:gbus_pgen_a0_s}-\ref{line:gbus_pgen_a0_e}). Otherwise, if the arity of the production rule is greater than $0$, and the cost of the rule is less than the desired cost $w(r) < c$, then \textsc{Next-Program} generates all programs with the target cost $c$ given by the rule $r$ with the parameters given by the Cartesian product of all programs in the bank $B$. Similarly to Algorithm~\ref{alg:bus}, \textsc{Next-Program} considers only the programs $p_1, \cdots, p_k$ from $B$ that are type-consistent with $r$. 
For example, when generating programs $p$ with rule $I \rightarrow \texttt{concat}(I, I)$, \textsc{Next-Program} only considers programs that return strings as subprograms of $p$. 


Once the program $p$ is yielded to the \textsc{Cost-Guided-BUS} procedure, the program is executed (line~\ref{line:gbus_execute}) and, if it satisfies the desired output $\mathcal{O}$, then it is returned as a solution to the task. Otherwise, the algorithm checks for observational equivalence, if $p$ is not observational equivalent to any other program in the bank $B$, then $p$ is stored in $B$ indexed by its cost $c$. 
Post-generation functions, such as the one used in \bustle, use a temporary cost $w(p)$ during the generation of programs $p$ in \textsc{Next-Program} (line \ref{line:gbus_pgen_a>0_s}) and assign a new cost $w'(p)$ once the program is generated (line \ref{line:gbus_reweight}). The new cost $w'(p)$ is then used to store the programs in $B$. Once the program $p$ is added to $B$, the main loop is repeated until the search finds the solution program $p$ or it exhausts the time allowed for synthesis. 

Algorithm~\ref{alg:guided-bus} generalizes \probe, \bustle, and \textsc{TF-Coder}. It is equivalent to \probe\ if it receives the cost function of \probe. We note that \probe\ learns a cost function during the search, while Algorithm~\ref{alg:guided-bus} assumes a fixed pre-generation cost function. It is equivalent to \bustle\ if it receives the post-generation cost function of \bustle. Finally, it is equivalent to \textsc{TF-Coder} if it receives \textsc{TF-Coder}'s hand-crafted cost function, which is also a pre-generation function.  


\subsubsection{Lack of Best-First Ordering for \probe\ and \bustle}


Both \textsc{Probe} and \textsc{Bustle} lose information because they round off the costs of the programs. As a result, the order in which they search over programs of a given cost is arbitrary, not best-first. In the PCFG given in Figure \ref{fig:pcfg}, \texttt{1000} has a lower cost than $ \texttt{ 1 }, \texttt{ 2 }, \cdots, \texttt{ 999}$, but when rounded, their costs become equal, i.e., $10$, and they are enumerated in an arbitrary order, while according to the cost function \texttt{1000} should be evaluated before \texttt{1}, \texttt{2}, $\cdots$, \texttt{999}. Since the number of programs with the same rounded cost can increase rapidly as the search grows, the algorithms' rounding off scheme can substantially slow down the synthesis process. 
In the example of synthesizing \texttt{concat(concat(1000, 1000), 1000)}, depending on how the ties are broken, \textsc{Probe} might evaluate more than 910,000,000 programs before finding the solution. By contrast, \textsc{Bee Search} evaluates 1,001,001 programs to find the solution. 

One could achieve a ``near best-first search'' if the costs were multiplied by a large constant (e.g., 1,000,000) before being rounded off. The issue with this approach is that, due to the large number of different costs, there would be many iterations of \textsc{Probe} and \textsc{Bustle} that no program would be generated (similar to how Algorithm~\ref{alg:guided-bus} did not generate any program with cost $[1, 9]$ in Example~\ref{example:generic_cost}); we refer to these iterations as \emph{sterile iterations}. \textsc{Probe} and \textsc{Bustle} still pay the computational cost of checking whether there are programs to be generated of a particular cost. Given that the target program cost of a given iteration is $c$ and that the production rule $r$ with $k$ non-terminal symbols costs $c'$, \textsc{Probe} checks all combinations of programs $(p_1, \cdots, p_{k})$ whose added cost is $c - c'$, so that the non-terminal symbols of $r$ can be replaced by $(p_1, \cdots, p_{k})$ and the total cost of the generated program is $c$; \textsc{Bustle} follows the same approach, but using its cost function. The task of finding a subset of numbers that adds to a target value is NP-Complete~\shortcite{NPComplete} and finding such subsets is exponential in the number of non-terminals $k$. Although $k$ can be small (e.g., for \texttt{concat} $k=2$), computing the subsets can still hamper the performance of the algorithm if the subsets have to be computed many times during the search. 

We performed preliminary experiments with a modified version of \textsc{Probe} that uses the ``large-constant trick'' and discovered that the approach is too slow to be practical: most of the computational effort is spent computing the subsets with target cost values for which no program can be generated. The algorithm we introduce in this paper, \bee, bypasses the NP-Complete problem of finding a subset of numbers that adds to a target value by performing a search in a cost-tuple space (see Section~\ref{section:bee-search} for details). 

\subsection{Heap Search}\label{sec:heap_search}

\begin{algorithm}[t]
    \caption{Heap Search Algortihm}
    \label{alg:heap_search}
    \begin{algorithmic}[1]
        \PROCEDURE \textsc{HeapSearch}($\mathcal{G}$, ($\mathcal{I}, \mathcal{O}$), $w$)
        \REQUIRE $\mathcal{G} = (V, \Sigma, R, I)$, input-output examples ($\mathcal{I}, \mathcal{O}$), and a cost function $w$.
        \ENSURE Solution program $p$ or $\perp$
        \FOR {\textbf{all} non-terminal symbols $T$ $\in$ $V$}\label{line:heaps_initialization_s}
        \STATE Create an empty min heap $\textsc{Heap}_T$
        \STATE Create an empty hash table $\textsc{Succ}_T$
        \STATE Create an empty set $\textsc{Seen}_T$
        \FOR {\textbf{all} derivation rules of $T$ $\rightarrow$ $p$ }
        \STATE Add $p$ to $\textsc{Heap}_{T}$ with priority $w(p)$
        \STATE Add $p$ to $\textsc{Seen}_{T}$
        \ENDFOR
        \FOR {\textbf{all} derivation rules $T$ $\rightarrow$ $r( T_1, \cdots, T_{k})$ }
        \STATE $p \gets r(\textsc{Heap}_{T_1}.\text{top()}, \cdots, \textsc{Heap}_{T_k}.\text{top()})$
        \STATE Add $p$ to $\textsc{Heap}_{T}$ with priority $w(p)$
        \STATE Add $p$ to $\textsc{Seen}_{T}$
        \ENDFOR
        \ENDFOR \label{line:heaps_initialization_e}
        \STATE $p \gets \emptyset$
        \STATE $T \gets I$
        \WHILE{not timeout}
            \STATE $p \gets \textsc{Query}(p, T)$
            \STATE $o \gets $ \textsc{Execute}$(p, \mathcal{I})$ 
            \IF{$o$ equals $\mathcal{O}$}
                \STATE \textbf{return} $p$ 
            \ENDIF
            \STATE $T \gets \text{type}(p)$
        \ENDWHILE
        \STATE \textbf{return} $\perp$ 
    \PROCEDURE \textsc{Query}({$p$, $T$}) \label{line:heap_query_s}
    \REQUIRE A program $p$, and type $T$ of the program.
    \ENSURE Next cheapest program $p'$
    \IF{$p$ is a key in $\textsc{Succ}_{T}$}\label{line:heap_succ_s}
    \STATE \textbf{return} $\textsc{Succ}_{T}$[$p$]\label{line:heap_succ_e}
    \ELSE
    \STATE $p' \gets pop(\textsc{Heap}_{T}$) \label{line:heapsearch_pop} \emph{\# $p'$ is of form $r(p_1, \cdots, p_k)$}
    \STATE $\textsc{Succ}_{T}[p] \gets p'$
    \FOR {\textbf{all} $i \in [1,\cdots, k]$}\label{line:heap_ch_gen_s}
    \STATE $y_i = \textsc{Query}(p_i, T_i)$ \label{line:heap_search_query_r_call}
    \STATE $p'_i = r(p_1, \cdots, p_{i-1}, y_i, p_{i+1}, \cdots, p_{k})$ \label{line:heap_search_child_generated}
    \IF{$p'_i$ is not in $\textsc{Seen}_{T}$} \label{line:heap_not_in_seen}
    \STATE Add $p'_i$ to $\textsc{Heap}_{T}$ with priority $w(p'_i)$
    \STATE Add $p'_i$ to $\textsc{Seen}_{T}$ \label{line:heap_adding_children}
    \ENDIF
    \ENDFOR\label{line:heap_ch_gen_e}
    \ENDIF
    \STATE \textbf{return} $p'$ \label{line:heap_query_e}
    \end{algorithmic}
\end{algorithm}

\heapsearch{} \cite{Fijalkow2022ScalingNP} performs best-first bottom-up synthesis with respect to a cost function $w$ defined with a PCFG. It achieves best-first enumeration by using a set of priority queues, which we denote as $\textsc{Heap}$, one for each non-terminal symbol $T$. We say that the programs derived from $T$ are of type $T$. Each queue in \textsc{Heap} contains programs of type $T$ that are sorted according to the programs' costs.  
\heapsearch{} also uses a set of programs already seen in search ($\textsc{Seen}$) and a hash table ($\textsc{Succ}$), to store the \emph{successors} of all programs of type $T$ seen in search. The successor of a program $p$, denoted $p'$, is the next cheapest program of type $T$ to be generated, i.e.,  
a program generated with a production rule for $T$ with $w(p') > w(p)$ such that there is no $p''$ with $w(p') > w(p'') > w(p)$. 



\begin{example}
Consider an example of \heapsearch{} using the following DSL. 
\begin{equation*}
    I \rightarrow 1 \,|\, 2 \,|\, I + I \,.
\end{equation*}
Here, $w(1) < w(2) < w(I \rightarrow I+I)$, and similarly to \probe, the cost of a program $p$ is given by the sum of the cost of the production rule and its subprograms $p_i$.
\heapsearch{} first generates programs \texttt{1}, \texttt{2}, and \texttt{1+1} (the cheapest program which can be generated using $I+I$) and add them to \textsc{Heap}$_I$, a heap structure storing programs of type $I$ (this DSL only has programs of type $I$). The programs are sorted according to their cost $w$. Then, it first evaluates \texttt{1} (cheapest program), followed by the next cheapest program, \texttt{2}. Once \texttt{2} is removed from the heap, \heapsearch{} sets \texttt{2} as the successor of \texttt{1} in the \textsc{Succ} hash table, i.e., $\textsc{Succ}[\texttt{1}]=\texttt{2}$. Next, it pops  \texttt{1+1} out and \texttt{1+1} is assigned as the successor of \texttt{2}. Since \texttt{1+1} was derived from a non-terminal production rule ($I \rightarrow I + I$), \heapsearch{} generates \texttt{1+1}'s children by replacing each subprogram $p$ of \texttt{1+1} with $p$'s successor. 
That is, it first replaces \texttt{1} (the first subprogram) with its successor (\texttt{2}) to generate \texttt{2+1}. Then, \heapsearch{} replaces the second subprogram with its successor and \texttt{1+2} is generated. Both of these programs are added to \textsc{Heap}$_I$. In the next iteration, \heapsearch{} pops out the next program from \textsc{Heap}$_I$ and continues the search until a solution program $p$ is removed from \textsc{Heap}$_I$  or it times out and it returns failure.  
\label{example:heap_search}
\end{example}



The pseudocode for \textsc{Heap Search}, adapted from \citeauthor{Fijalkow2022ScalingNP}~\citeyear{Fijalkow2022ScalingNP}, is shown in Algorithm \ref{alg:heap_search}. The algorithm starts by initializing all data structures \textsc{Heap}$_T$, \textsc{Seen}$_T$, and \textsc{Succ}$_T$ with the programs given by the terminal symbols $p$ of $\mathcal{G}$. The structures are also initialized with the cheapest program of each type $T$, which is given by production rules $T \rightarrow r(T_1, \cdots, T_k)$. Each subprogram of type $T_i$ in $r(T_1, \cdots, T_k)$ is given by the cheapest program generated with a terminal symbol of type $T_i$. For example, for type $T_i$ we use $\textsc{Heap}_{T_i}.top()$ as all heaps are already initialized by the terminal symbols. In our example, the rule $I \rightarrow I + I$ generated the program \texttt{1 + 1} because program \texttt{1} was the cheapest program generated with a terminal symbol. 

\heapsearch{} invokes \textsc{Query} while there is still time allowed for synthesis. \textsc{Query} receives a program $p$ and its type $T$ as input; it returns the successor of $p$. \heapsearch{} initially calls \textsc{Query} with an empty program and the initial symbol of the grammar, $I$. Then, it calls \textsc{Query} with the program it returned on its last call. Each call to \textsc{Query} returns the next program according to the best-first ordering of the programs given by $w$. Each program \textsc{Query} returns is evaluated and, if it represents a solution, the program is returned; \heapsearch{} returns failure, $\perp$, if it times out before finding a solution. 

The \textsc{Query} procedure returns the successor of the program $p$ passed as input and it recursively generates the children of $p$'s successor. The base case of the recursion is when the successor of $p$ is already stored in \textsc{Succ}$_T[p]$ (lines~\ref{line:heap_succ_s} and \ref{line:heap_succ_e}). If the successor $p'$ is not in \textsc{Succ}$_T$, then it is removed from \textsc{Heap}$_T$ and $p'$ is set as the successor of $p$ in \textsc{Succ}$_T$. In Example~\ref{example:heap_search}, the successor of \texttt{2} is \texttt{1 + 1}; the latter was popped out of \textsc{Heap}$_T$ when \textsc{Query} was invoked for \texttt{2}. \textsc{Query} then generates all children of $p'$, by replacing each of $p'$'s $k$ subprograms with the successors of its subprograms. The subprogram successors are obtained by calling \textsc{Query} recursively (line~\ref{line:heap_search_query_r_call}). In our example, the children of \texttt{1 + 1} were \texttt{2 + 1} and \texttt{1 + 2}. The subprogram \texttt{2} of \texttt{2 + 1} was obtained by calling \textsc{Query} with the program \texttt{1}; program \texttt{2} was returned as the base case of the recursion as \textsc{Succ}$_T[\texttt{1}] = \texttt{2}$.  

\heapsearch{} only inserts a newly generated program if it has not been added to a \textsc{Heap} before. This is achieved by storing in the hash table \textsc{Seen} all programs generated in search (lines~\ref{line:heap_not_in_seen}--\ref{line:heap_adding_children}). Note that \heapsearch{} only does not re-insert in a \textsc{Heap} the exact programs that were seen before. This is different from the equivalence check \bus\ algorithms perform. While \bus\ algorithms would disregard the program \texttt{1 + 1} because it is observational equivalent to \texttt{2}, \heapsearch{} considers both programs in search. \heapsearch{} provably evaluates programs in a best-first ordering according to a pre-generation cost function.

\subsubsection{Limitations of \heapsearch{}}

\heapsearch{} sacrifices the ability of remove observational equivalent programs to attain best-first ordering with respect to a pre-generation cost function. If \heapsearch{} performed equivalence check, it would no longer be a complete algorithm as it could fail to find a solution even for solvable problems. In our example, \texttt{1 + 1} would be eliminated as it is observational equivalent to \texttt{2} and \heapsearch{} would not generate the children of \texttt{1 + 1}, which cannot be generated through another branch of the search. 

Moreover, \textsc{Heap Search} is guaranteed to \emph{evaluate} programs in a best-first order, but it does not \emph{generate} programs in the best-first order. Whenever \textsc{Query} is called, it returns a single program that is evaluated, however, it generates many other programs (lines \ref{line:heap_ch_gen_s}-\ref{line:heap_ch_gen_e}) that might never be evaluated because they are more expensive than the cost of a solution program.

In addition, \textsc{Heap Search} was not designed to search with post-generation functions such as $w_{\bustle}$. If the algorithm is modified to handle post-generation cost functions, then it would not be able to search in a best-first order as its proof implicitly assumes a pre-generation cost function.

\subsection{Brute Search} \label{sec:brute_search}

\begin{algorithm}[t]
    \caption{Brute Search}
    \label{alg:brute}
    \begin{algorithmic}[1]
    \PROCEDURE \textsc{Brute}($\mathcal{G}, (\mathcal{I}, \mathcal{O}), w$)
    \REQUIRE $\mathcal{G} = (V, \Sigma, R, I)$, input-output examples ($\mathcal{I}, \mathcal{O}$), and a cost function $w$.
    \ENSURE Solution program $p$ or $\perp$
    \STATE $n_0 \gets \{T \vert T \in \Sigma \wedge I \rightarrow T \}$ \label{line:brute_root}
    \IF{a program $p$ in $n_0$ is a solution}
    \STATE \textbf{return} $p$
    \ENDIF
    \STATE Add $n_0$ to $Q$ with priority $0$ (highest priority possible)
    \STATE B = \{\}
    \WHILE{not timeout}\label{line:brute_main_loop_s}
    \STATE{$n \gets $ $Q$.pop()} \emph{\# node $n$ can represent one or multiple programs} \label{line:brute_popping}
    \FOR{each child program $p$ of $n$} \label{line:brute_ch_s} 
    \STATE{$o \gets $\textsc{Execute}$(p, \mathcal{I})$} 
    \IF{$o$ equals $\mathcal{O}$}
    \STATE \textbf{return} $p$ \label{line:brute_return}
    \ENDIF
    \IF{$o$ not in $B$} \label{line:brute_observational_equivalent}
    \STATE $B$.add$(p,o)$ \label{line:brute_add_to_bank}
    \STATE Add node representing $p$ to $Q$ with priority $w(p)$ \label{line:brute_add_to_queue}
    \ENDIF
    \ENDFOR 
    \ENDWHILE \label{line:brute_main_loop_e}
    \STATE \textbf{return} $\perp$ 
    \end{algorithmic}
\end{algorithm}

\brute{} is a best-first search algorithm we adapted to inductive program synthesis that is loosely inspired by the inductive logic programming system of the same name~\cite{brute}. Let us consider an example with the DSL from Example \ref{example:heap_search}: $I \rightarrow 1 \,|\, 2 \,|\, I + I$. 

In \brute, the root of the tree represents all programs given by symbols appearing in terminal production rules. In our example, the root represents the programs \texttt{1} and \texttt{2}. The root is the only node representing multiple programs; all other nodes in the tree represent one program. The set of children of a node $n$ in the \brute{} search tree is defined as follows. For each non-terminal production rule we generate all possible programs given by the Cartesian product of all programs seen in search where at least one of the subprograms of the children is given by a program $n$ represents. In our example, the children of the root are given by the Cartesian product of programs \texttt{1} and \texttt{2} with rule $I \rightarrow I+I$: \texttt{1 + 1}, \texttt{1 + 2}, \texttt{2 + 1}, and \texttt{2 + 2}. The children nodes representing programs \texttt{1 + 1} and \texttt{2 + 1} are pruned because they are observational equivalent to \texttt{2} and \texttt{1 + 2}, respectively. The next layer of the tree is generated following the same procedure. For example, the children of the node representing \texttt{1 + 2} are given by the Cartesian product of all programs observed in search as subprograms of the production rule $I \rightarrow I+I$, where at least one subprogram is \texttt{1 + 2}. The children of the node representing \texttt{1 + 2} are: \texttt{1 + (1 + 2)}, \texttt{2 + (1 + 2)}, \texttt{(1 + 2) + 1}, \texttt{(1 + 2) + 2}, \texttt{(1 + 2) + (2 + 2)}, \texttt{(2 + 2) + (1 + 2)}; after pruning observational equivalent programs, we obtain: \texttt{2 + (1 + 2)} and \texttt{(1 + 2) + (2 + 2)}.  

The pseudocode of \brute{} is shown in Algorithm \ref{alg:brute}. The root of the tree, $n_0$, is defined as the set of programs given by the terminal rules (line \ref{line:brute_root}); if any of these programs $p$ represents a solution, then $p$ is returned. Otherwise, $n_0$ is added to a priority queue $Q$. 
In every iteration of the algorithm, \brute{} pops the cheapest node $n$ from $Q$ (line~\ref{line:brute_popping}) and generates its children, as illustrated in our example above. Each child of $n$ represents a program $p$. If $p$ is a solution, then it is returned. Otherwise, if $p$ is not observational equivalent to another program in $B$ (line~\ref{line:brute_observational_equivalent}), then (i) \brute{} adds $p$ to the bank of programs $B$ (line \ref{line:brute_add_to_bank}) and (ii) a node representing $p$ is added to the queue with priority $w(p)$ (line~\ref{line:brute_add_to_queue}).

\subsubsection{Limitations of \brute{}} 

Similarly to \heapsearch{}, \brute{} only evaluates programs in best-first order and does not generate programs in best-first order. In each iteration, it evaluates a single program but possibly generates many more. In \brute{} the branching factor can be very large compared to \heapsearch{} since it considers all the programs evaluated so far in the Cartesian product used to generate the children of a node. This can substantially slow down the synthesis process and increase its memory usage because it generates many programs that will never be evaluated, as they can be more expensive than the solution program.

\section{Best-First Bottom-Up Search (\textsc{Bee Search})} \label{section:bee-search}

\bee\ attains best-first ordering with respect to the generation of programs by performing a search in a cost-tuple space, which is explained in the next section. 

\subsection{Cost-Tuple Space}

\textsc{Bee Search} searches over a set of cost-tuple spaces to determine the next program to be generated during search. We define one cost-tuple space for each non-terminal rule. A state in a cost-tuple space is defined by a tuple with $k$ integers in $\mathbb{N}$ (we use $1$ as the index of the first element of an array), where $k$ is the number of non-terminal symbols in the production rule. For example, $(i_1, i_2)$ represents a state in the cost-tuple space of the rule $I \rightarrow $ \texttt{concat}$(I, I)$; $i_1$ and $i_2$ represent indexes in an ordered set $C$ that contains the cost of all programs generated in search and it is sorted from the smallest to the largest cost. We use the words `state' and `cost-tuple state' interchangeably. Since each state is related to a production rule $r$, we denote by $n.r$ the production rule associated with $n$. 

Each cost-tuple state represents a set of programs that are to be generated. For example, the state $(1, 1)$ of the space for $I \rightarrow \texttt{concat}(I, I)$ represents all programs that can be generated by replacing the first and second non-terminal symbols of \texttt{concat} by the cheapest programs encountered in search. According to the costs of the PCFG shown in Figure~\ref{fig:pcfg},  \texttt{1000} is the program with the lowest $w$-value encountered in the space defined by the DSL ($w$-value of $9.8165$). Initially, $C = \{9.8165\}$ and the cost-tuple state $n = (1, 1)$ for $I \rightarrow \texttt{concat}(I, I)$ represents the program \texttt{concat(1000, 1000)} and the cost-tuple state's $w$-value can be computed as $w(n) = 14.28771 + 9.8165 + 9.8165 = 33.92071$. For additive cost functions, the $w$-value of the state is equal to the $w$-value of the programs that the state represents. Although only \texttt{1000} costs $9.8165$, 
each state $n = (i_1, i_2, \cdots, i_k)$ can represent multiple programs, as there might be multiple programs with the $i$-th cost in $C$. 

\bee\ uses a priority queue $Q$ that is initialized with one cost-tuple state $(1, \cdots, 1)$ for each non-terminal rule, where the size of the tuple matches the number of non-terminal rules on the right-hand side of the rule. Each cost-tuple state represents a set of programs with a given cost $w$; therefore, the priority queue is sorted according to the $w$-value of each cost-tuple state. In every iteration, \bee\ pops the cheapest cost-tuple state $n = (i_1, i_2, \cdots, i_k)$ from $Q$ and generates all programs $n$ represents. If none of these programs represents a solution to the program synthesis task, then the children of $n$ are generated and inserted in $Q$. The children of $n$ are given by the set of states that differ from $n$ with the addition of $1$ to an entry of $n$: $\{(i_1 + 1, i_2, \cdots, i_{k}), (i_1, i_2+1, \cdots, i_{k}), \cdots, (i_1, i_2, \cdots, i_{k} + 1)\}$. We say that a cost-tuple state $n$ is expanded when its children are generated. 

Let us consider the following example. 

\begin{example}
Table \ref{table:bfs-running-example} shows a trace of \bee\ for the problem of synthesizing the program  {\texttt{concat(1000, concat(1000, 1000))}}. In this example, we will consider the cost function $w_{\probe}$ described in Figure~\ref{fig:pcfg}. The table shows updates to the \bee's cost list $C$ and priority queue $Q$, programs generated, and the number of programs generated in each iteration (Count). The entries in column $Q$ are cost-tuples of the form [cost, cost-tuple]; costs are truncated to four decimal places and the name of production rule is omitted from the cost-tuple for brevity. The number of entries in a cost-tuple indicates the arity of the production rule; the empty cost-tuple $(\,)$ represents all programs generated with a terminal rule and cost-tuples with two entries, e.g., $(1, 1)$, represent states for \texttt{concat}.

\begin{table*}[h]
    \footnotesize
    \centering
    \begin{tabular}{ccccc}
    \toprule
    Itr. \# & $C$ & $Q$ & Programs & Count  \\ \midrule
    $1$ & $\{9.8165\}$  & \makecell{$\{[9.9660, (\,)],[33.9207, (1,1)]\}$} & \texttt{1000} & $1$  \\ \midrule
    $2$ & $\{9.8165, 9.9660\}$ & $\{[33.9207, (1,1)]\}$ & \texttt{1}, \texttt{2}, $\cdots$, \texttt{999} & $999$  \\ 
    \midrule
    $3$ & $\makecell{\{9.8165, 9.9660,\\ 33.9207\}}$ & $\makecell{\{[34.0702, (2,1)], [34.0702,(1,2)]\}}$ & \texttt{concat(1000,1000)} & $1$  \\ 
    \midrule
     $4$ & $\makecell{\{9.8165, 9.9660,\\ 33.9207, 34.0702\}}$ & \makecell{ $\{ [34.2197, (2,2)], [58.0249, (1,3)],$ \\ $[58.0249, (3,1)] \}$ }
     &  \makecell{\texttt{concat(1000, 1)}, \\ $\ldots$ \texttt{concat(999, 1000)}} & $1998$ 
     \\ 
     \midrule
     $5$ & $\makecell{\{9.8165, 9.9660,\\ 33.9207, 34.0702,\\ 34.2197 \}}$ & \makecell{$\{[58.0249, (1,3)], [58.0249, (3,1)],$ \\ $[58.1744, (3,2)]\}$} & \makecell{\texttt{concat(1, 1)}, \\  \ldots \texttt{concat(999, 999)}} & $998001$  \\ \midrule
     $6$ & \makecell{$\{9.8165, 9.9660,$ \\ $33.9207, 34.0702,$ \\ $34.2197, 58.0249\}$} & \makecell{$\{[58.1744, (3,2)], [58.1744, (2,3)],$ \\ $[58.1744, (4,1)],[58.1744, (1,4)]\}$} & \makecell{\texttt{concat(1000, }\\ \texttt{concat(1000, 1000))}} & $1$ \\ \bottomrule
    \end{tabular}
    \caption{The table shows the trace of \bee\ for finding {\texttt{concat(1000, concat(1000, 1000))}} with the cost function described in Figure~\ref{fig:pcfg}.  The table shows the cost list $C$, priority queue $Q$, the generated programs, and the count of programs generated at each iteration.}
    \label{table:bfs-running-example}
\end{table*}

\bee{} starts by initializing $C$ with the cost of the cheapest program generated with a terminal rule. $Q$ is initialized with empty cost-tuple states, with one state for each terminal rule that is not the cheapest, and with one cost-tuple state for each non-terminal rule. In this example, all programs generated with production rules $I \rightarrow j$ for $j \in \{\texttt{1}, \texttt{2}, \cdots, \texttt{999}\}$ are represented in state $[9.9660, ()]$, while state $[33.9207, (1,1)]$ represents the program \texttt{concat(1000, 1000)}.\footnote{\bee\ would actually insert one cost-tuple state for each program generated with a terminal rule: $[9.9660, (), \texttt{1}], [9.9660, (), \texttt{2}], \cdots, [9.9660, (), \texttt{999}]$. For the interest of space we represent all these cost-tuple states as $[9.9660, ()]$ in this example. \bee\ would also spend one iteration with each of these states, which we also simplify in this example to a single iteration.} 
In the first iteration, \bee\ pops the cheapest node of $Q$, which represents the program \texttt{1000}, whose cost is $9.8165$. In the second iteration, it removes the state $[9.9660, ()]$ and generates 999 programs \texttt{1}, $\cdots$, \texttt{999} with cost $9.9660$; both costs are added to $C$.

Next, the search removes $n = [33.9207, (1,1)]$ from $Q$ and generates \texttt{concat(1000, 1000)} with a cost of $33.9207$. Here, $1$ in the cost-tuple refers to the first index of the cost set $C$ and the only program we have with that cost is \texttt{1000}; $(1, 1)$ indicates that both arguments of \texttt{concat} should be of cost $9.8165$, hence the program \texttt{concat(1000, 1000)} whose cost equals the sum of the costs of the subprograms \texttt{1000} and the cost of the operation \texttt{concat}. Since $n$ represents a non-terminal rule, the node is expanded, thus generating the cost-tuple states $\{[34.0702, (2,1)], [34.0702, (1,2)]\}$. Similarly, in fourth and fifth iterations, the search generates all programs with cost $34.0702$ and $34.2197$, respectively. After generating all programs of each node, it adds their children to $Q$. \bee\ finds the solution program with cost $58.0249$ in iteration 6. Note that there are other programs with similar cost (i.e., $58.1744$) and \bee\ is able to distinguish them from the solution node. The programs with cost $58.1744$ are not even generated in memory, but only their cost-tuple states. 
\label{example:running-example-sol-bee-search}
\end{example}

\bee's ability to distinguish programs with similar cost values (e.g., $58.0249$ and $58.1744$) can make a substantial difference in terms of search running time. As an example, \probe{} is unable to distinguish $58.0249$ (used in the sixth iteration of the example) and $58.1744$ (the next cheapest cost) as both values are truncated to $58$. As a result, \probe{} evaluates approximately one billion programs to find the solution for our example. By contrast, \bee\ evaluates only approximately one million programs to find the solution. 

Further, the search in the cost-tuple space allows for a best-first search with respect to the generation of programs. At each iteration, \bee\ only generates the programs with the next cheapest possible cost. Unlike other best-first search algorithms such as \brute{} and \heapsearch, \bee\ does not have to generate the programs to identify the ones that will be evaluated next in the search. This is achieved at the cost of generating cost-tuple states that might never be evaluated in search (i.e., cost-tuple states for which we do not generate their programs). This feature is an advantage because the branching factor in the cost-tuple space is much smaller than the branching factor in the program space. We show empirically the advantages of performing best-first search with respect to the generation of programs. 

\subsection{Search Algorithm}
\label{Section:SearchAlgos}

\begin{algorithm}[t]
\caption{\textsc{Bee Search}}
\label{alg:bbus}
\begin{algorithmic}[1]
\PROCEDURE \textsc{BeeSearch}($\mathcal{G}, (\mathcal{I}, \mathcal{O}), w$)
\REQUIRE $\mathcal{G} = (V, \Sigma, R, I)$, input-output examples ($\mathcal{I}, \mathcal{O}$), and a monotonically increasing cost function $w$.
\ENSURE Solution program $p$ or $\perp$
\STATE $C \gets \{\min_{T \in \Sigma} w(T)\}$
\STATE $Q \gets \emptyset$ \emph{\# Sorted according to the $w$-values}
\FOR{each non-terminal rule $S$ in $\mathcal{G}$}
\STATE $Q$.push$([S, (1, \cdots, 1)])$ \label{line:init_q_1}
\ENDFOR
\FOR{each terminal symbol $S$ in $\mathcal{G}$}
\STATE $Q$.push$([S])$ \label{line:init_q_2}
\ENDFOR
\WHILE{not timeout} \label{line:main_loop}
\STATE $p, c \gets $ \textsc{Next-Program}$(B, Q)$
\STATE $o \gets $ \textsc{Execute}$(p, \mathcal{I})$
\IF{$o$ equals $\mathcal{O}$}\label{line:bee_solution_check}
\STATE \textbf{return} $p$ \label{line:return_p}
\ENDIF
\STATE $B[c]$.add$(p)$ \label{line:add_bank}
\ENDWHILE
\STATE \textbf{return} $\perp$ \label{line:failure}
\end{algorithmic}
\end{algorithm}

\begin{algorithm}[h]
\caption{\textsc{Next-Program} procedure}
\label{alg:next}
\begin{algorithmic}[1]
\PROCEDURE \textsc{Next-Program($B, Q$)}
\REQUIRE Bank of programs $B$, priority queue $Q$
\ENSURE Next cheapest program $p$ of cost $c$
\STATE $n \gets Q.$pop() \label{line:pop}
\IF{pre-generation cost function} \label{line:pre_1}
\STATE $C \gets C \cup w(n)$ \label{line:pre_2}
\ENDIF
\IF{$n$ represents a terminal symbol}
\STATE \textbf{return} $n.r$, $w(n)$ \label{line:return_terminal}
\ENDIF
\FOR{$i$ in $\{1, \cdots, k\}$} \label{line:expand_1}
\STATE $n' \gets n$
\STATE $n'[i] \gets n'[i] + 1$
\IF{$n'$ is not a duplicate}
\STATE $Q$.push($n'$) \label{line:expand_2}
\ENDIF
\ENDFOR
\FOR{$(p_1, \cdots, p_{k})$ \textbf{in} $B[C[n[1]]] \times \cdots \times B[C[n[k]]]$} \label{line:bbus_cartesian}
\IF{$(p_1, \cdots, p_k)$ is type-consistent with $n.r$}
\STATE $p \gets n.r(p_1 \cdots, p_{k})$ \label{line:bbus_new_programs}
\IF{$p$ is equivalent to any program in $B$} \label{line:equivalence_check_1}
\STATE \textbf{continue} \label{line:equivalence_check_2}
\ENDIF
\IF{post-generation cost function}
\STATE Insert $w'(p)$ in $C$ while keeping $C$ sorted \label{line:post}
\IF {$w'(p) < C[i]$, where $i = \max_j \{j \in n \vert n \in Q\}$} \label{line:reheapify_1}
\STATE Heapify $Q$ \label{line:reheapify_2}
\ENDIF
\STATE \textbf{yield} $p$, $w'(p)$ \label{line:return_1}
\ELSE
\STATE \textbf{yield} $p$, $w(p)$ \label{line:return_2}
\ENDIF
\ENDIF
\ENDFOR
\end{algorithmic}
\end{algorithm}

Algorithms~\ref{alg:bbus} and \ref{alg:next} show the pseudocode for \textsc{Bee Search}. The algorithm receives a DSL, a set of input-output examples, and a monotonically increasing cost function $w$; it returns a solution $p$ or failure $\perp$. The search starts by adding in $C$ the $w$-value of the cheapest program generated with a terminal rule and initializing a priority queue $Q$ with one state $(1, \cdots, 1)$ for each non-terminal rule (line~\ref{line:init_q_1} of Algorithm~\ref{alg:bbus}). $Q$ also receives one state for each terminal rule; these states do not have a tuple associated with them (line~\ref{line:init_q_2}) because they do not generate children in the cost-tuple space. The ordering of $Q$ is defined by the $w$-values of the cost-tuple states. 

In every iteration of the algorithm (iteration of the while loop in Algorithm~\ref{alg:bbus}), it generates the next program $p$ (see Algorithm~\ref{alg:next}), which is executed with the input values $\mathcal{I}$, thus obtaining the outputs $o$. 
If $o$ matches the desired output $\mathcal{O}$, then $p$ is a solution, and it is returned (line~\ref{line:return_p}). If $p$ is not a solution, it is added to the bank of programs $B$, which is indexed by the cost of the programs $c$ (line~\ref{line:add_bank}). This process continues until either a solution is found or the search times out, in which case failure $\perp$ is returned (line~\ref{line:failure}). 

Algorithm~\ref{alg:next} defines the program that is evaluated next in search. We remove a state $n$ with the smallest $w$-value from $Q$ (line~\ref{line:pop}).
If the state represents a terminal rule, then the state has no children (i.e., it does not have non-terminal symbols that can be used to generate new programs). In this case, we just return the program $n.r$ and its cost $w(n)$ (line~\ref{line:return_terminal} of Algorithm~\ref{alg:next}), where function $r$ returns the righthand side of the rule state $n$ represents. 
We expand $n$ if it represents a non-terminal rule, i.e., we generate all children $n'$ of $n$ and we add them to $Q$ if they were not already inserted in $Q$
(lines~\ref{line:expand_1}--\ref{line:expand_2}). 

If $n$ represents a non-terminal rule, we generate the set of programs $n$ represents and we return each program $p$ as an iterator to Algorithm~\ref{alg:bbus}.\footnote{An iterator is implemented with the keyword ``yield'' and it allows a function $f_1$ return a value to a function $f_2$ and, once $f_2$ calls $f_1$ again, the execution in $f_1$ continues from the last ``yield'' executed.}  The programs $n$ represents are generated by replacing each non-terminal symbol of $n$'s rule with a program from $B$. Let $n[j]$ be the $j$-th value of $n$, the $j$-th non-terminal symbol of $n.r$ is replaced by a program with cost $C[n[j]]$. Given that $B$ is indexed by the programs' cost, we obtain the programs $(p_1, \cdots, p_{k})$ that replace the non-terminals of $n.r$ by taking the Cartesian product $B[C[n[1]]] \times \cdots \times B[C[n[k]]]$ (line~\ref{line:bbus_cartesian}). Similarly to Algorithms~\ref{alg:bus} and \ref{alg:guided-bus}, \bee\ performs type checking while generating programs with the Cartesian product operation. 
We denote newly generated programs as $n.r(p_1, \cdots, p_{k})$. Since all programs $B[C[n[i]]]$ have the same cost, the $w$-value of all programs $n.r(p_1, \cdots, p_{k})$ matches the $w$-value of $n$. We discard observational equivalent programs (lines~\ref{line:equivalence_check_1} and \ref{line:equivalence_check_2}).

\textsc{Bee Search} treats post-generation and pre-generation functions differently. For pre-generation functions, the $w$-value of all programs generated from state $n$ is identical to the $w$-value of $n$. Thus, the cost of the programs can be added to the set $C$ before generating them (lines~\ref{line:pre_1} and \ref{line:pre_2}). For post-generation functions, the $w'$-values are known only after generating the programs, so they are inserted in $C$ in line~\ref{line:post}. While for pre-generation functions the cost values are inserted in increasing order in $C$ and they can simply be appended at the end of $C$, the costs $w'$ are not necessarily generated in increasing order (programs generated from the same tuple $n$ can have different $w'$-values). \textsc{Bee Search} maintains $C$ sorted with post-generation functions by inserting the $w'$-values in their correct positions in $C$. Let $i$ be the largest index in a cost-tuple state in $Q$. If the $w'$-value inserted in $C$ is smaller than the $i$-th value in $C$ (denoted $C[i]$ in the pseudocode), then $Q$ needs to restore its heap structure through a ``heapify'' operation. This is because, once the $w'$-value is inserted in $C$, some of the indexes $j$ in the cost-tuple states in $Q$ might not refer to the same $C[j]$ they referred to prior to the insertion of the new $w'$-value. In the worst case, the insert operation is linear in the size of $C$. The heapify operation is more expensive because it is linear in the size of $Q$ and $Q$ tends to be much larger than $C$. However, in practice, the heapify operation is rarely performed. By keeping $C$ sorted and $Q$ with a valid heap structure, we can prove that \bee\ is a best-first search algorithm also for post-generation functions (see Section~\ref{sec:guarantees}). 
Each program and its cost are returned as an iterator in lines~\ref{line:return_1} and \ref{line:return_2}. If the $w'$ function is encoded in a neural network, instead of evaluating one program at a time, we evaluate all programs generated from $n$ in a batch for efficiency; the batch evaluation is not shown in the pseudocode. 

\subsection{Theoretical Guarantees}
\label{sec:guarantees}

\bee\ is complete, correct, and performs best-first bottom-up search with respect to an additive cost function $w$ that can be either of the pre-generation or the post-generation type. In this section, we provide the proofs for these properties.

\begin{lemma}
Let $w$ be an additive cost function. If the values in $C$ are sorted from smallest to largest, then the $w$-values of the cost-tuple states increase monotonically in \bee. 
\label{lemma:monotone_state}
\end{lemma}
\begin{proof}
The children $n'$ of a state $n = (i_1, i_2, \cdots, i_k)$ are identical $n$, except for one entry $j$ in $n$ that is incremented by $1$. For additive functions we have $w(n) = K + \sum_{i = 1}^k C[i]$, where $K$ is the cost of the production rule $n$ represents. Then, we have the following.
\begin{align*}
w(n) &= K + \sum_{i = 1}^k C[i] \\
&= K + C[j] + \sum_{\substack{i = 1 \\ i \neq j}}^k C[i] \\
&< K + C[j + 1] + \sum_{\substack{i = 1 \\ i \neq j}}^k C[i]\\ 
&= w(n') \,.
\end{align*}
The inequality is due to $C$ being sorted from smallest to largest and the values in $C$ being unique. 
\end{proof}

The following theorems state that \textsc{Bee Search} performs a best-first search with respect to the generation of programs for a family of $w$ functions that includes $w_{\textsc{Probe}}$ (Theorem~\ref{theorem:pre}) and for a family of $w$ functions that includes $w_{\textsc{Bustle}}$ and $w_{\textsc{u}}$ (Theorem~\ref{theorem:post}).
\begin{theorem}
\textsc{Bee Search} generates programs in best-first order with respect to an additive pre-generation cost function $w$.
\label{theorem:pre}
\end{theorem}

\begin{proof}
We prove by induction in the iterations of search that \textsc{Bee Search} expands all cost-tuple states $n$ in best-first order with respect to $w$, which implies that the programs $p$ generated from $n$ are evaluated in best-first order because $w$ is additive and thus $w(p) = w(n)$ for all $p$. The base case is \textsc{Bee Search}'s first iteration when $C$ is initialized with the $w$-value of a terminal symbol with the smallest $w$ value. Since $w$ is additive, no state $n$ can have a value of $w$ lower than $C[1]$. The inductive hypothesis states that all cost-tuple states up to the $j$-th expansion (excluding the $j$-th expansion) are processed in best-first order with respect to $w$. Since $w$ is a pre-generation function, the cost values $w(n)$ are added to $C$ once a cost-tuple state $n$ is expanded, so $C$ must be sorted in increasing order prior to the $j$-th expansion. 

For the inductive step, we consider the $j$-th expansion. 
In the $j$-th expansion, \bee\ expands a cost-tuple state $n_1$ with the smallest $w$-value in $Q$. Let us suppose that there is another state $n_2$ that was not expanded before $n_1$ and $w(n_2) < w(n_1)$ (i.e., \bee\ would have to expand $n_2$ instead of $n_1$ to attain best-first ordering). 
Since $C$ is sorted from smallest to largest (inductive hypothesis) and $w(n) < w(n_2) < w(n_1)$ for any ancestor $n$ of $n_2$ (Lemma~\ref{lemma:monotone_state}), either $n_2$ or one of its ancestors $n$ would have the smallest $w$-value in $Q$ and not $n_1$, which is a contradiction, since $n_1$ has the smallest $w$-value in $Q$. Thus, the search expands the cost-tuple states in best-first order with respect to $w$, which implies that it generates programs in best-first order with respect to $w$ ($w(n) = w(p)$ for all programs $p$ generated from $n$). 
\end{proof}

\begin{theorem}
\textsc{Bee Search} generates programs in best-first order with respect to an additive and penalizing post-generation cost function $w$.
\label{theorem:post}
\end{theorem}
\begin{proof}
During a \bee\ search with penalizing post-generation cost functions, the values of $C[i]$ for a fixed $i$ can change across iterations due to the penalization term of $w'$ and the sorting \bee\ performs. We prove by induction in the iterations of the search that, at the time of expansion of a cost-tuple state $n = (i_1, i_2, \cdots, i_k)$, $C[i]$ has its minimum value for all indexes $i$ in $n$. By proving that the $C[i]$-entries have their minimum values, we show that the $C[i]$-values cannot change later in search for all $i$ in $n$. Since \bee\ maintains $C$ sorted and the $C[i]$-values cannot change, we can use Lemma~\ref{lemma:monotone_state} to show that the cost-tuple expansions happen in best-first order with respect to $w$, which implies a best-first order for the generation of programs as $w(n) = w(p)$ for all $p$ generated from $n$. In our proof we consider cost-tuple states representing non-terminal rules since states representing terminal rules do not generate children, and thus the priority queue alone guarantees the best-first ordering for such states. 
The base case is the first cost-tuple state $n = (1, \cdots, 1)$ expanded. Since $C[1]$ contains the cost of the terminal symbol with the smallest $w$-value and the $w$ function is additive, no other program can have a smaller $w$-value. The inductive hypothesis states that the $C[i]$-values have their minimum value and are sorted for all indexes $i$ of states \textsc{Bee Search} expands up to the $j$-th expansion. 

Let $n_1$ be the cost-tuple state with the smallest $w$-value in $Q$ in the $j$-th expansion. 
Let us suppose that, at a given iteration, due to the order in which \bee\ inserts cost values in $C$, there is an $i$ in $n_1$ for which $C[i]$ does not have its minimum value. That is, there exists a cost-tuple state $n_2$ that was not expanded yet and that will generate a program $p$ whose $w'(p)$-value will be assigned to $C[i]$, and before this assignment happens, we have $C[i] > w'(p)$. 
Penalizing cost functions guarantee that the value of a program cannot be smaller than the value of the cost-tuple state that generated the program, i.e., $w'(p) \geq w(n_2)$, for a $p$ generated from $n_2$.  
Since $w(n_1)$ is given by the sum of costs of its subprograms ($w$ is additive) and $w'(p)$ is one of the terms of the sum that results in $w(n_1)$, then
$w'(p) < w(n_1)$ and thus $w(n_2) < w(n_1)$. Since \bee\ always maintains $Q$ with a valid heap structure and the cost function increases monotonically for the cost-tuple states (inductive hypothesis and Lemma~\ref{lemma:monotone_state}), $n_2$ and all its ancestors must have been expanded prior to $n_1$ and the $C[i]$-values for all $i$ in $n_1$ must be at their minimum value when $n_1$ is expanded. 

Since the $C[i]$-values are sorted and final for all $i$ in the states \bee\ expands, Lemma~\ref{lemma:monotone_state} gives us that the cost-tuple states are expanded in best-first order according to $w$, which implies that the programs are generated in best-first order according to $w$. 
\end{proof} 

The next theorem shows that \bee\ search is complete, i.e., if there is a solution is the space of programs $\mathcal{G}$ defines, \bee\ will eventually find it. 

\begin{property} Given enough memory and time, if a solution program $p$ exists in the search space defined by the grammar $\mathcal{G}$, then \bee\ will find it.
\begin{proof} 
\bee\ considers all possible cost-tuple states in the search---it does not leave any state unchecked. Since every program that can be derived from $\mathcal{G}$ is mapped to a cost-tuple state, \bee\ considers all possible programs during search. 
\end{proof}
\label{property:completeness}
\end{property}

We begin by discussing the correctness of \bee\ by showing that all indexes stored in cost-tuple states refer to valid positions in the cost set $C$, despite $C$ being dynamically constructed during search. This means that \bee\ never accesses an index that is outside the range of $[1, \vert C \vert]$, and thus it does not throw a runtime error for that reason.

\begin{property}
 For monotonically increasing cost functions, the indexes $i_j$ in the cost-tuples $(i_1, i_2, \cdots, i_{k})$ generated during the \textsc{Bee Search} search are valid, i.e., $i_j$ in $[1, |C|]$.
\label{property:valid_indexes}
\end{property}
\begin{proof}
The proof is by induction in the iterations of search. $C$ is initialized with the cost of the cheapest terminal symbol, so all tuples $(1, \cdots, 1)$ refer to a valid index. The inductive hypothesis is that prior to the $j$-th iteration of \textsc{Bee Search} all indexes $i$ in all tuples $(i_1, i_2, \cdots, i_{k})$ generated thus far in search are valid. In the $j$-th iteration of \textsc{Bee Search} the cost-tuple state $n$ is to be expanded and, by the inductive hypothesis, all its indexes are valid, including the largest index, denoted $i_{\textsc{max}}$. Since all indexes are valid, we know that $|C| \geq i_{\textsc{max}}$. If $|C| > i_{\textsc{max}}$, then all children of $n$ are trivially valid because each index in $n$ grows at most $1$ in $n$'s children. If $|C| = i_{\textsc{max}}$, all children are also valid because when $n$ is expanded, its cost is added to $C$, thus increasing the size of $C$ by $1$. The cost $C[i_{\textsc{max}}]$ is the largest in $C$ before $n$ is expanded. Since the cost function is increasing monotonically, $w(n) > C[i_{\textsc{max}}]$ (or $w'(p) > C[i_{\textsc{max}}]$, where $p$ is a program generated from $n$, for post-generation cost functions). Thus, when $n$ is expanded, the size of $C$ increases by $1$ and all children of $n$ have valid indexes. 
\end{proof}
The following theorem states that the program \bee\ returns is correct, i.e., it solves the program synthesis task. 
\begin{property}
If \bee\ returns a solution program $p$, then $p$ is correct as it solves the program synthesis task: $p$ satisfies the semantic and syntactic constraints of the task.
\begin{proof}
It is trivial to establish that \bee\ is correct, the program $p$ it returns must satisfy the semantic and syntactic constraints of the synthesis problem, as \bee checks for these constraints in line~\ref{line:bee_solution_check} of Algorithm~\ref{alg:bbus} and only returns the solution program $p$ (line~\ref{line:return_p}) if the constraints are satisfied. 
\end{proof}
\label{theorem:correctness}
\end{property}

\section{Empirical Results}\label{section:results}

We evaluate \bee\ on three benchmark problems set: (i) $205$ string manipulation problems---$108$ programming by example string problems from $2017$ SyGuS competition, $37$ real problems faced by people and posted on StackOverflow, and $60$ spreadsheet problems from Exceljet~\cite{Woosuk2018} (we call this benchmark the SyGuS benchmark), (ii) $38$ handcrafted string manipulation problems from \textsc{Bustle}'s paper~\cite{bustle}, and (iii) $27$ bit-vector problems from the Hacker's Delight book~\shortcite{hackersdelightbook}. 
We have implemented \textsc{Probe}, \textsc{Bustle}, \textsc{Brute} with $w_\textsc{Probe}$ and $w_\textsc{Bustle}$, \textsc{Bee Search} with $w_\textsc{Probe}$, $w_\textsc{Bustle}$, and $w_\textsc{U}$, \textsc{Heap Search} with $w_\textsc{Probe}$, and \bus. 

\textsc{Probe} uses an online learning scheme in which the probabilities of the PCFG are updated as more input-output examples are solved. \textsc{Probe} uses a parameter to determine when to update the probabilities; we tested the values of $d = \{1, 2, \cdots, 7\}$ in all algorithms using $w_\probe$, and report the results for the value that performed best for each algorithm. \textsc{Bustle} uses a neural network with property signatures~\shortcite{Odena2020Learning}. 
Property signatures are domain dependent and \citeauthor{bustle}~\citeyear{bustle} described properties only for the string manipulation domain, so we limit the experiments with techniques using $w_{\textsc{Bustle}}$ and $w_{\textsc{u}}$ to string manipulation problems only. However, we evaluate the approaches using $w_\textsc{Probe}$ on both string and bit-vector problems. We report the average and standard deviation over $5$ independent runs of all results involving the neural network of \textsc{Bustle}. \bus{} is deterministic, so is \probe's learning scheme, hence we report the results of a single run for them.

We are interested in comparing \bee\ using $w_{\textsc{Probe}}$ with all other algorithms using $w_{\textsc{Probe}}$ (\probe, \brute, and \heapsearch) and \textsc{Bee Search} using $w_{\textsc{Bustle}}$ with all other algorithms using $w_{\textsc{Bustle}}$ (\bustle\ and \brute). We are also interested in comparing all algorithms performing best-first search: \bee{}, \brute{}, and \heapsearch{}. Lastly, we are interested in comparing \textsc{Bee Search} using $w_{\textsc{u}}$ with all other algorithms. We performed two sets of experiments: one with a smaller DSL and another with a larger one (see Appendix \ref{appendix:DSL} for the DSLs). 
The larger DSLs are defined as follows. For each problem in the string domain of SyGuS benchmarks, instead of using only the literals given in the problem's specification, we use literals from all problems in the set and all letters in the English alphabet. For the $38$ string manipulation problems, in addition to the aforementioned literals, $50$ randomly generated strings of length selected uniformly at random from the range $[2, 7]$, and $10$ integers selected uniformly at random from the range $[10, 100]$ are also added. For the bit-vector domain, we used $250$ literals for all problems, obtained by taking the union of literals from all problems in the set and adding other $236$ random literals. The goal of experimenting with larger DSLs is to evaluate the algorithms in larger search spaces. Larger DSLs also simulate scenarios in which one does not have access to the set of literals required to solve a problem, and more literals can increase the chances of defining spaces that contain a solution. All experiments were run on $2.4$ GHz CPUs with $16$ GB of RAM. The algorithms had $120$ minutes to solve each task. 

Figures~\ref{fig:baseline-dsl-strings-results} and \ref{fig:38-strings-results} present the results for the string manipulation tasks from $205$ SyGuS competition and $38$ handcrafted benchmarks, respectively. Figure~\ref{fig:bit-vector-results} shows the results of bit-vector domain. In each figure, the two plots at the top present the results for the smaller DSL, while the plots at the bottom present the results for the larger DSL. We present the number of problems solved by the number of programs evaluated and by the running time in seconds. While running time offers a fair evaluation of the different algorithms, the number of evaluations is machine- and implementation-independent, which allows it to be more easily used by others. 
The plots were generated by sorting the solved instances according to each algorithm's running time (or number of evaluations); the y-axis shows the total number of problems solved and the x-axis the sum of running times (or sum of number of evaluations). 

\subsection{Discussion}

\begin{figure*}[!t]
    \centering
    \begin{subfigure}[b]{0.495\textwidth}
        \includegraphics[width=\linewidth]{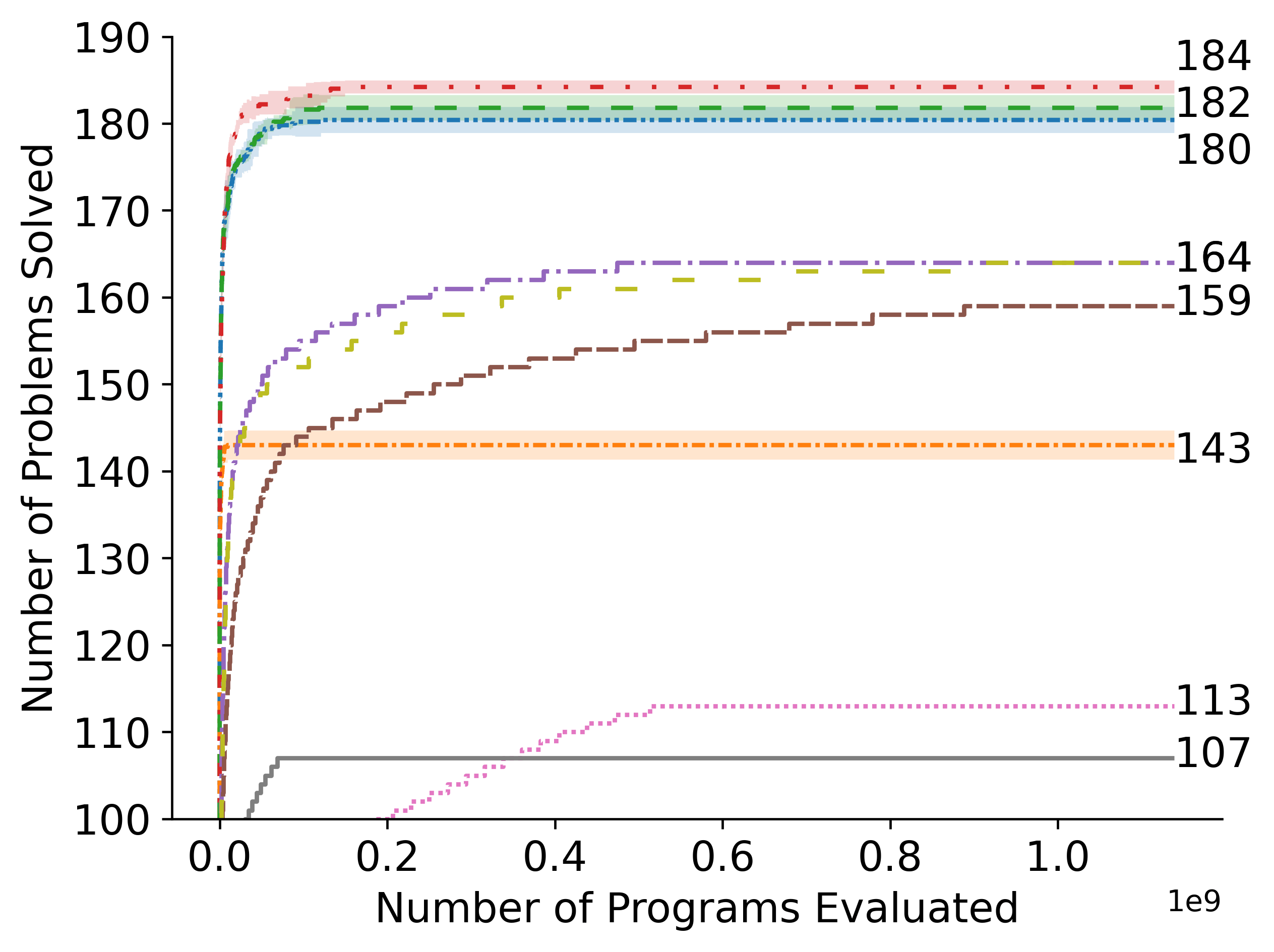}
    \end{subfigure}  \hfil
    \begin{subfigure}[b]{0.495\textwidth}
        \includegraphics[width=\linewidth]{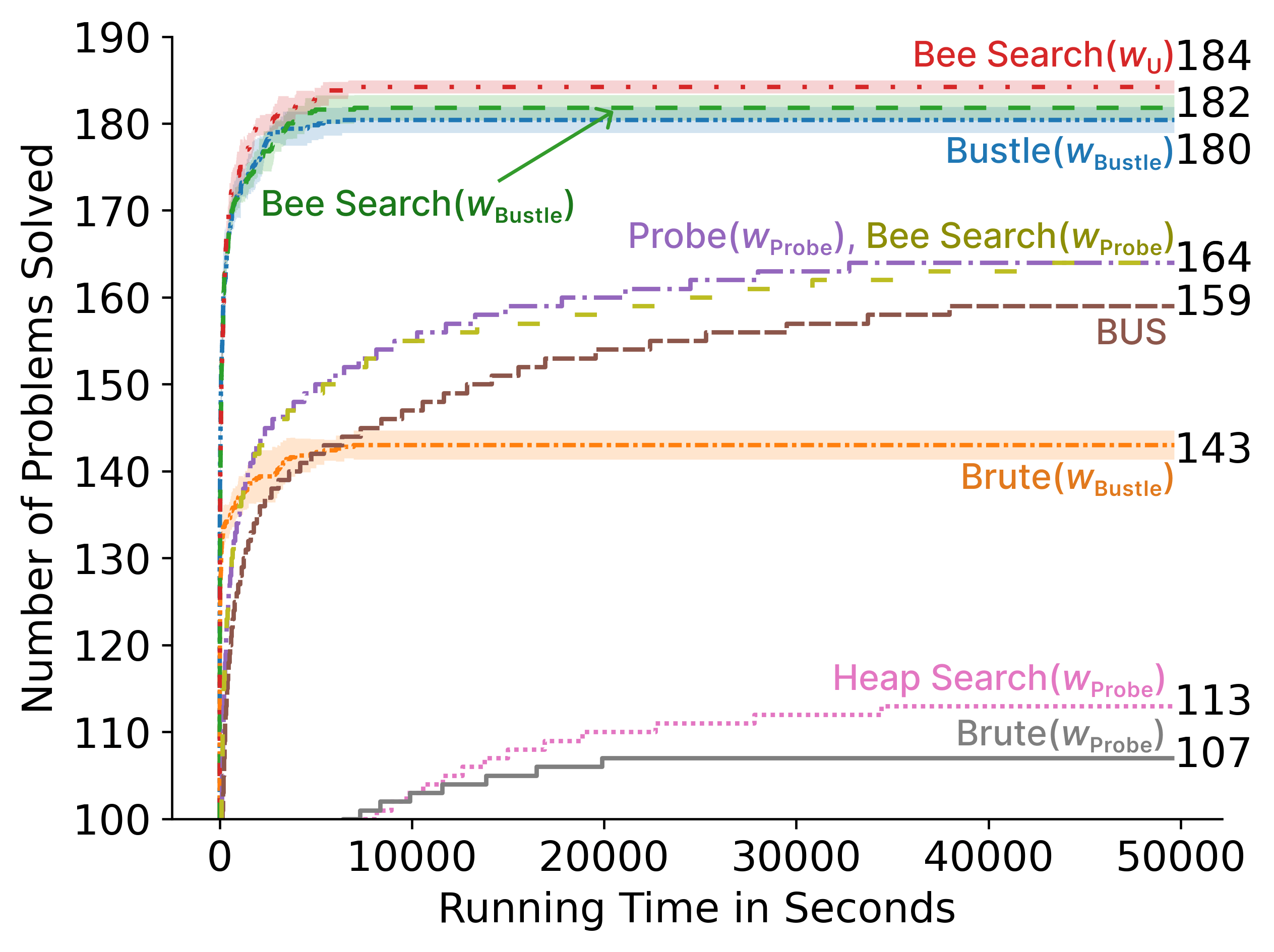}
    \end{subfigure}

    \medskip

    \begin{subfigure}[b]{0.495\textwidth}
        \includegraphics[width=\linewidth]{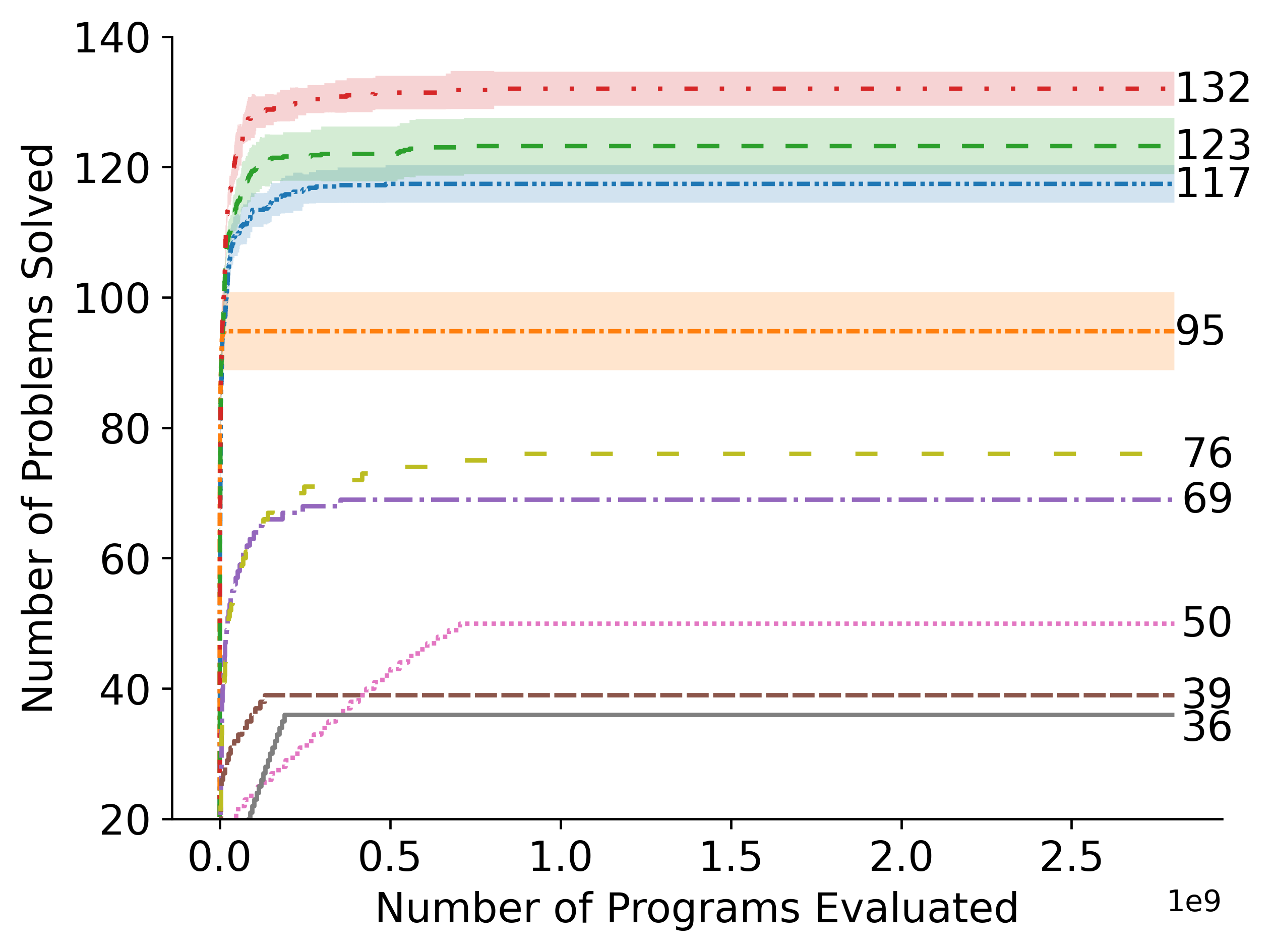}
    \end{subfigure}  \hfil
    \begin{subfigure}[b]{0.495\textwidth}
        \includegraphics[width=\linewidth]{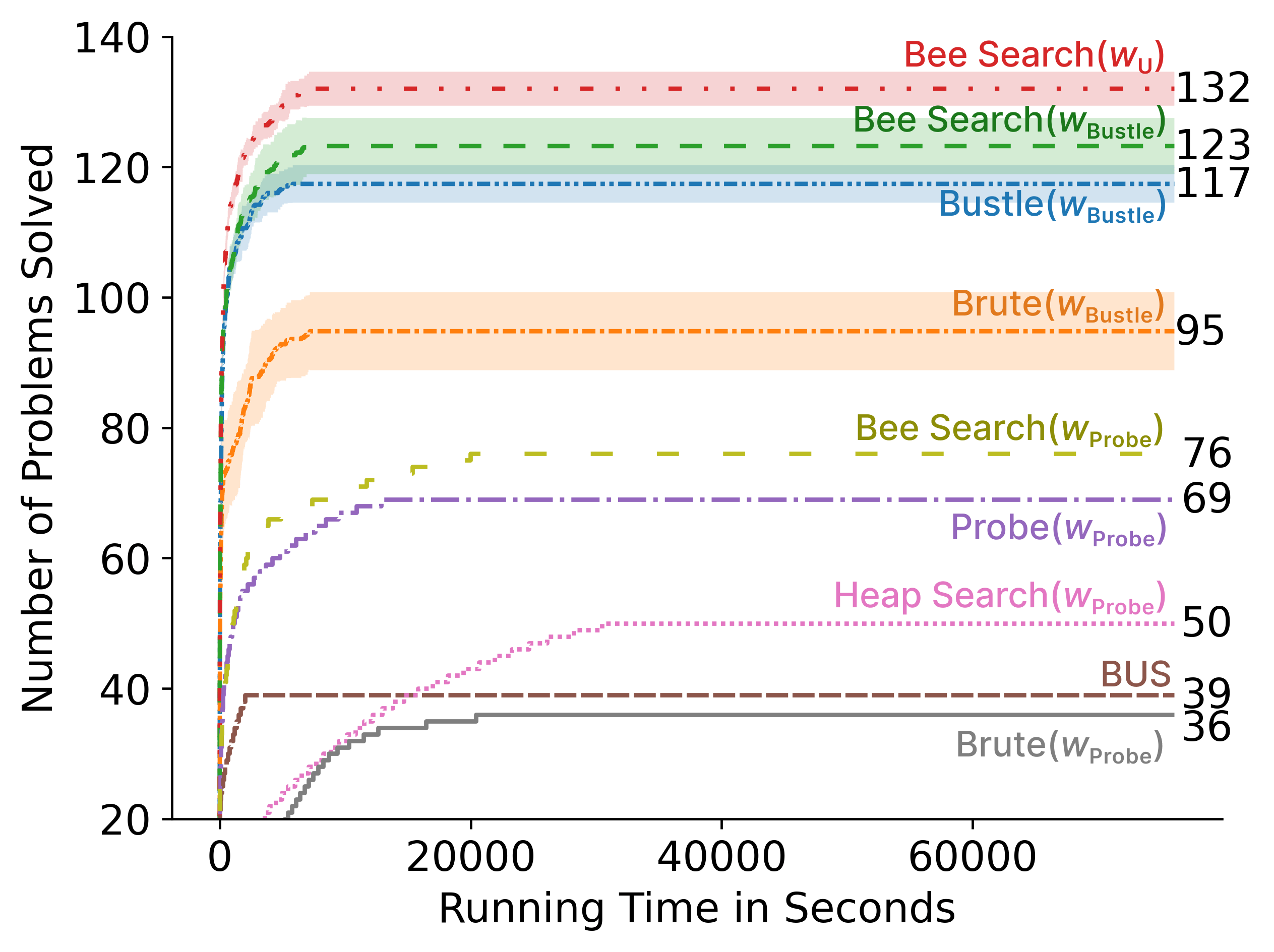}
    \end{subfigure}
\caption{Number of problems solved per number of evaluations and running time for 205 string domain problems of SyGuS. The two plots at the top show the results for the smaller DSL; the plots at the bottom show the results for the larger DSL.}
\label{fig:baseline-dsl-strings-results}
\end{figure*}

\begin{figure*}[t]
    \centering
    \begin{subfigure}[b]{0.495\textwidth}
        \includegraphics[width=\linewidth]{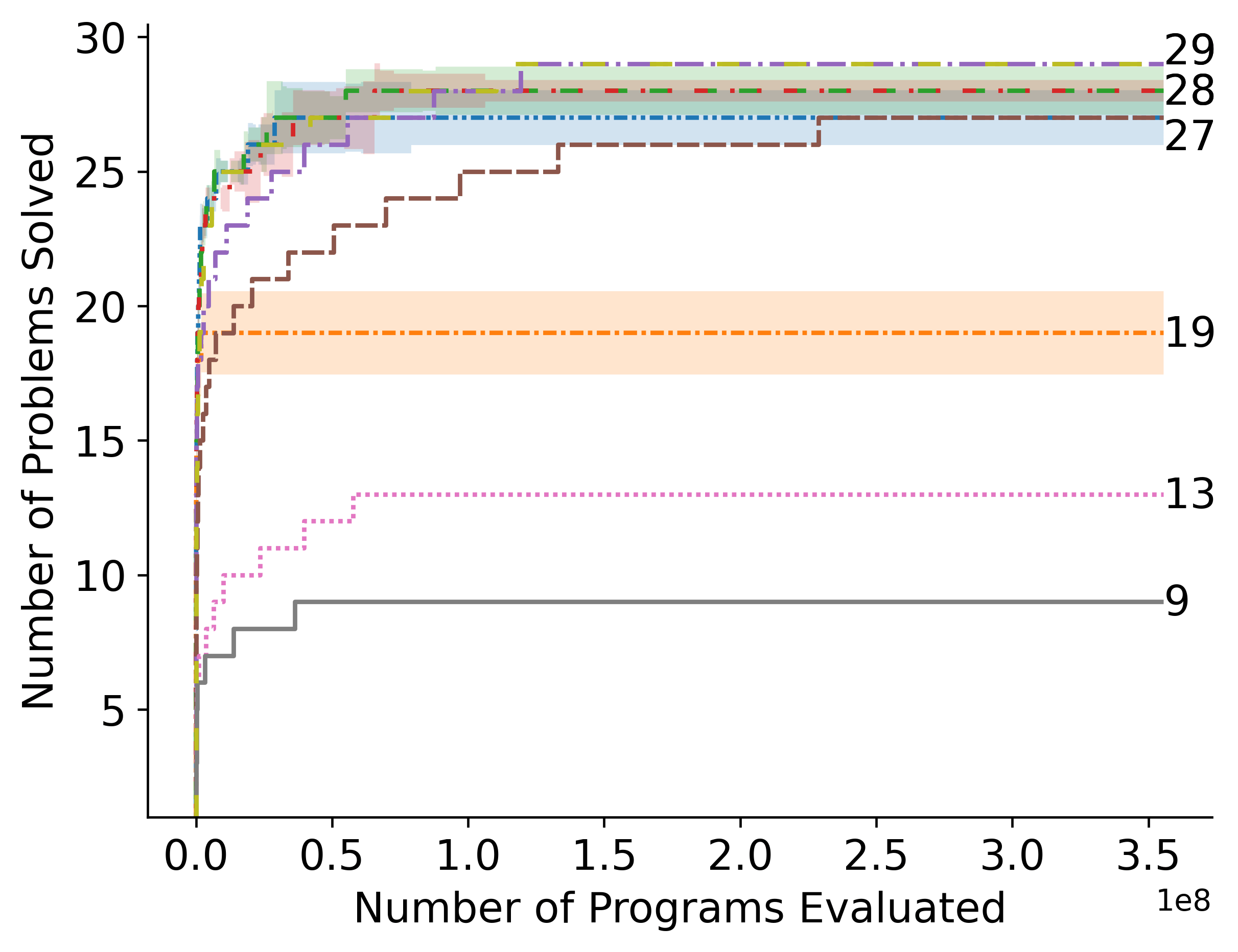}
    \end{subfigure}  \hfil
    \begin{subfigure}[b]{0.495\textwidth}
        \includegraphics[width=\linewidth]{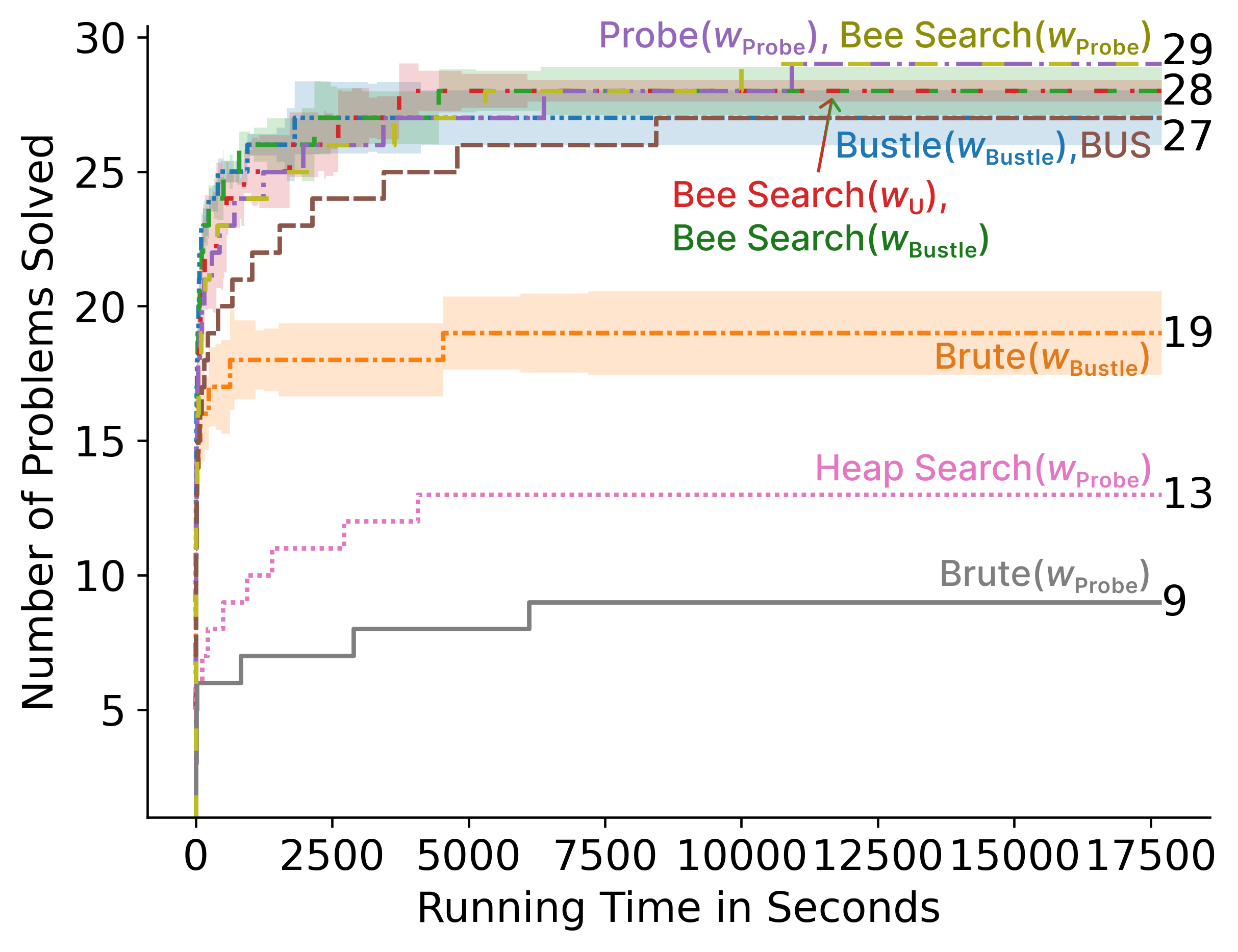}
    \end{subfigure}

    \medskip


    \begin{subfigure}[b]{0.495\textwidth}
        \includegraphics[width=\linewidth]{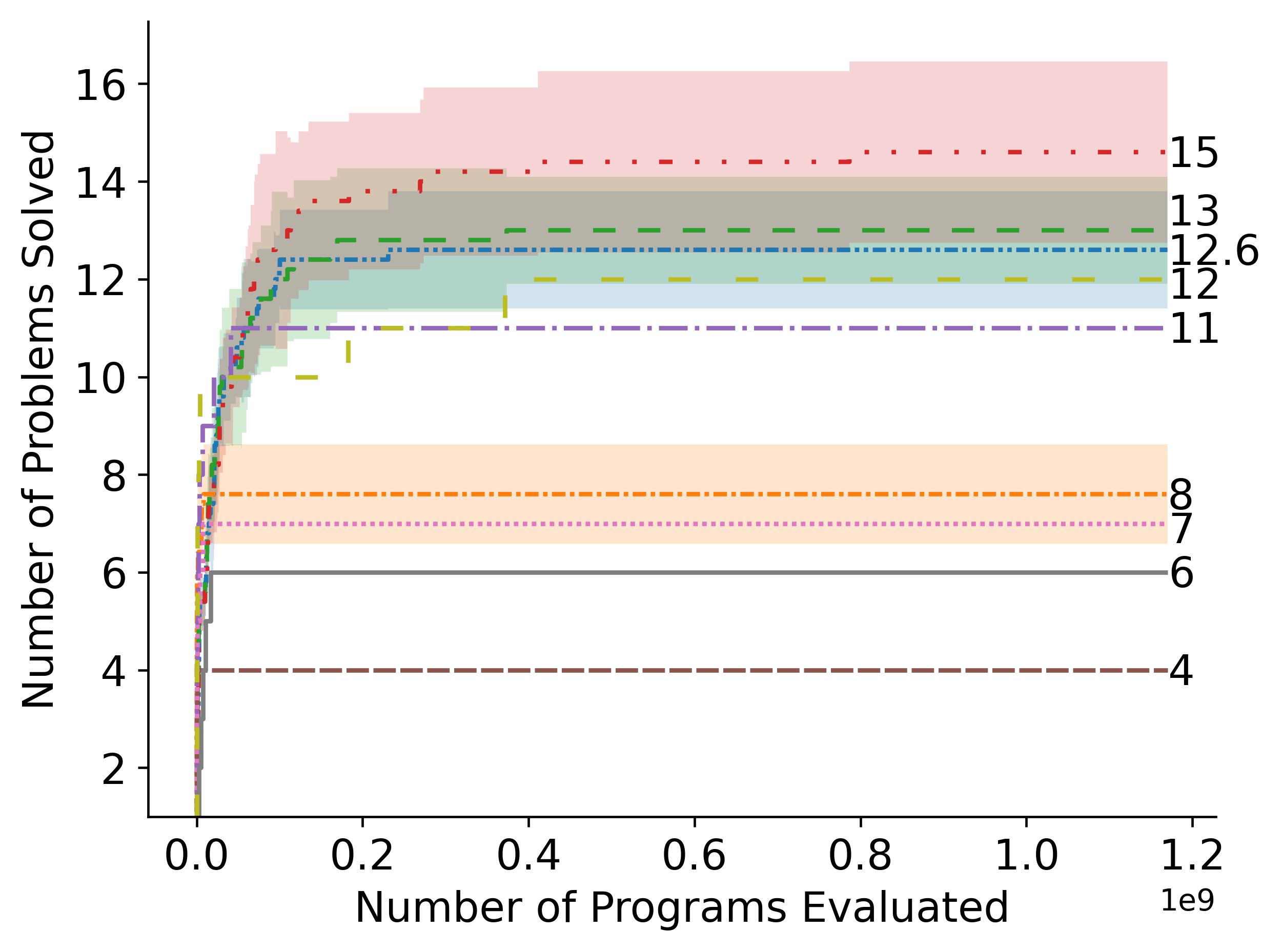}
    \end{subfigure}  \hfil
    \begin{subfigure}[b]{0.495\textwidth}
        \includegraphics[width=\linewidth]{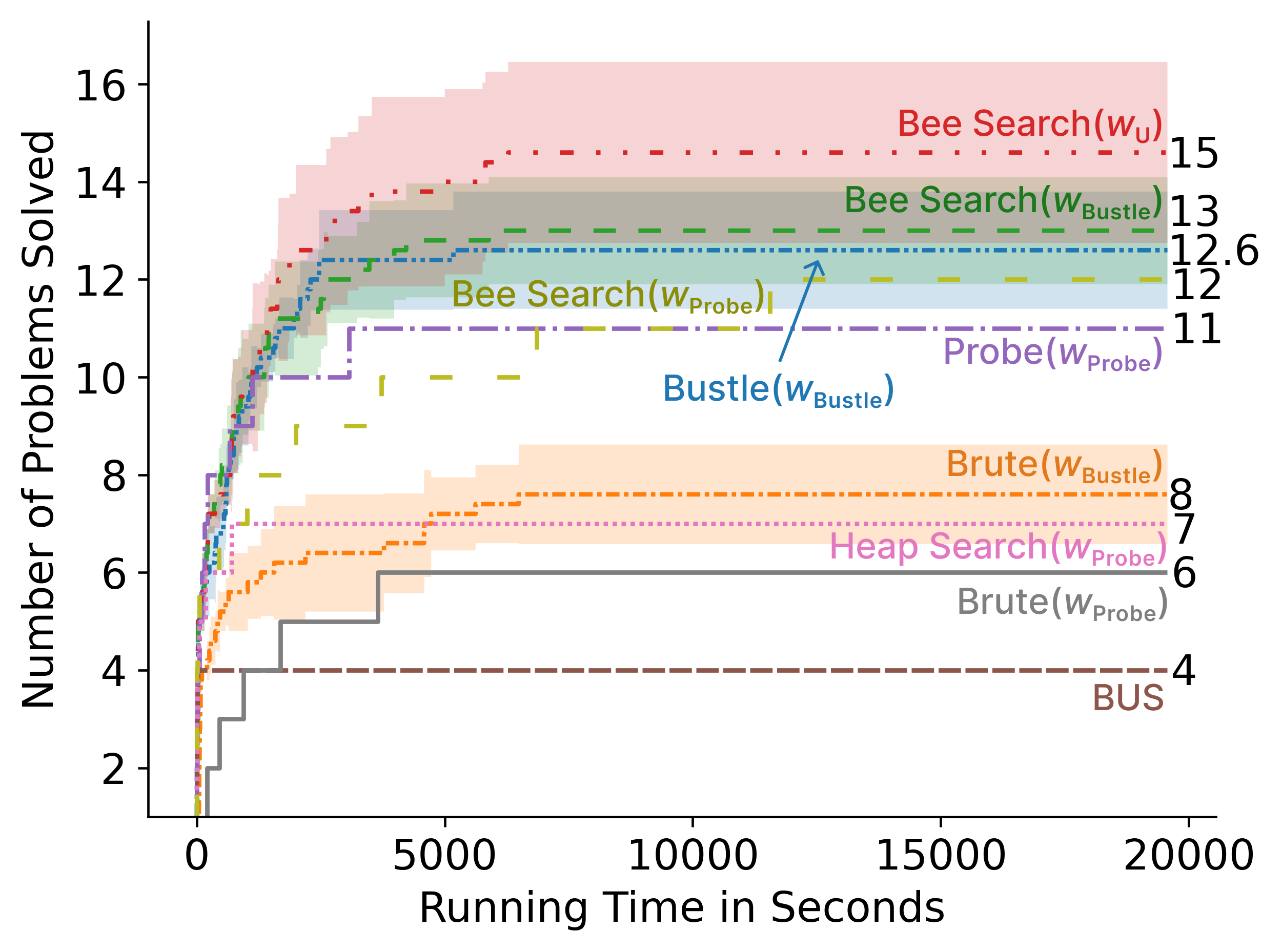}
    \end{subfigure}
    
\caption{Number of problems solved per number of evaluations and running time for the handcrafted 38 strings benchmarks by \citeauthor{bustle}. The two plots at the top show the results for the smaller DSL; the ones at the bottom show the results for the larger DSL.}
\label{fig:38-strings-results}
\end{figure*}

\begin{figure*}[htb]
\centering
\begin{subfigure}[b]{0.495\textwidth}
              \includegraphics[width=\linewidth]{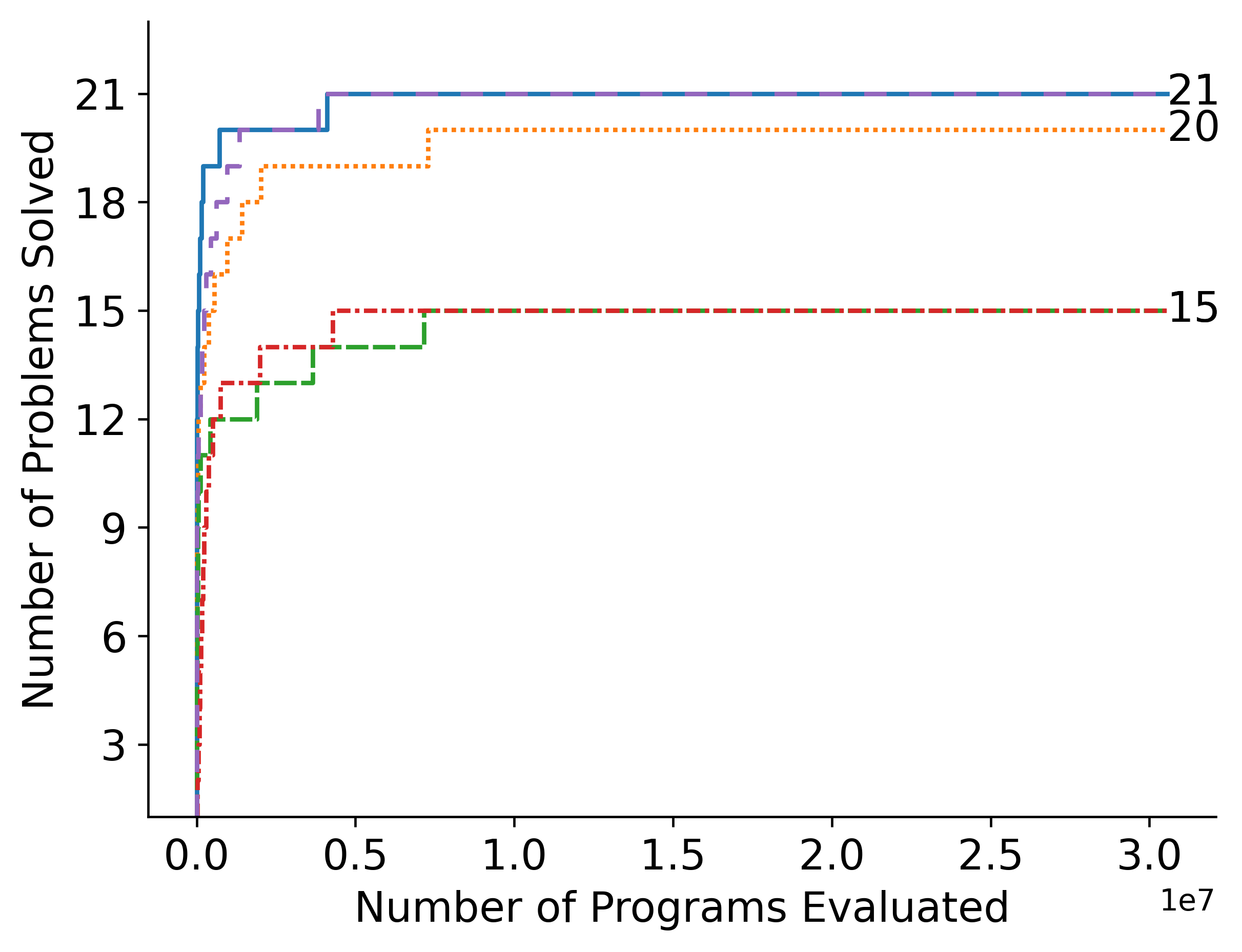}
\end{subfigure}  \hfil
       \begin{subfigure}[b]{0.495\textwidth}
              \includegraphics[width=\linewidth]{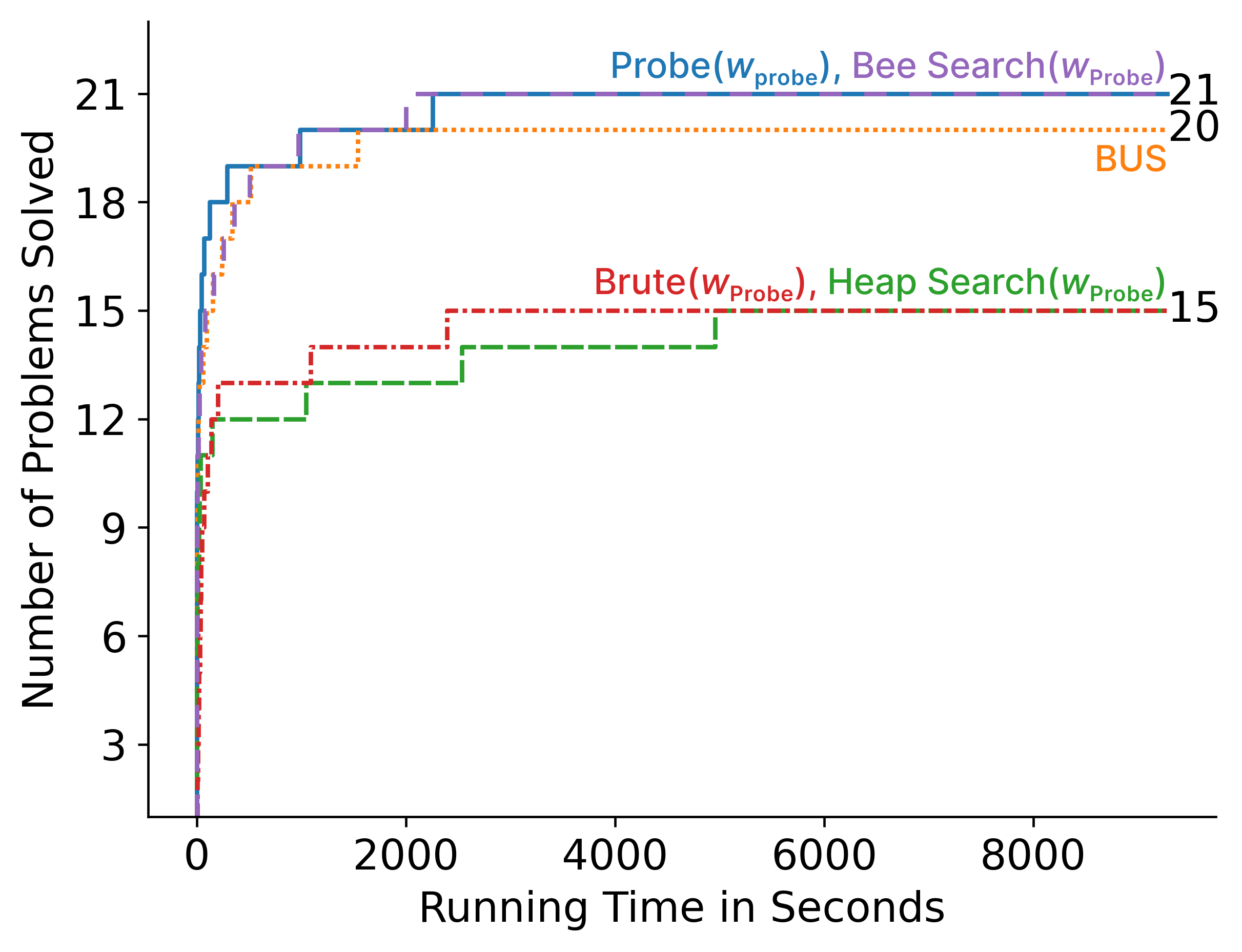}
\end{subfigure}
\medskip


\begin{subfigure}[b]{0.495\textwidth}
       \includegraphics[width=\linewidth]{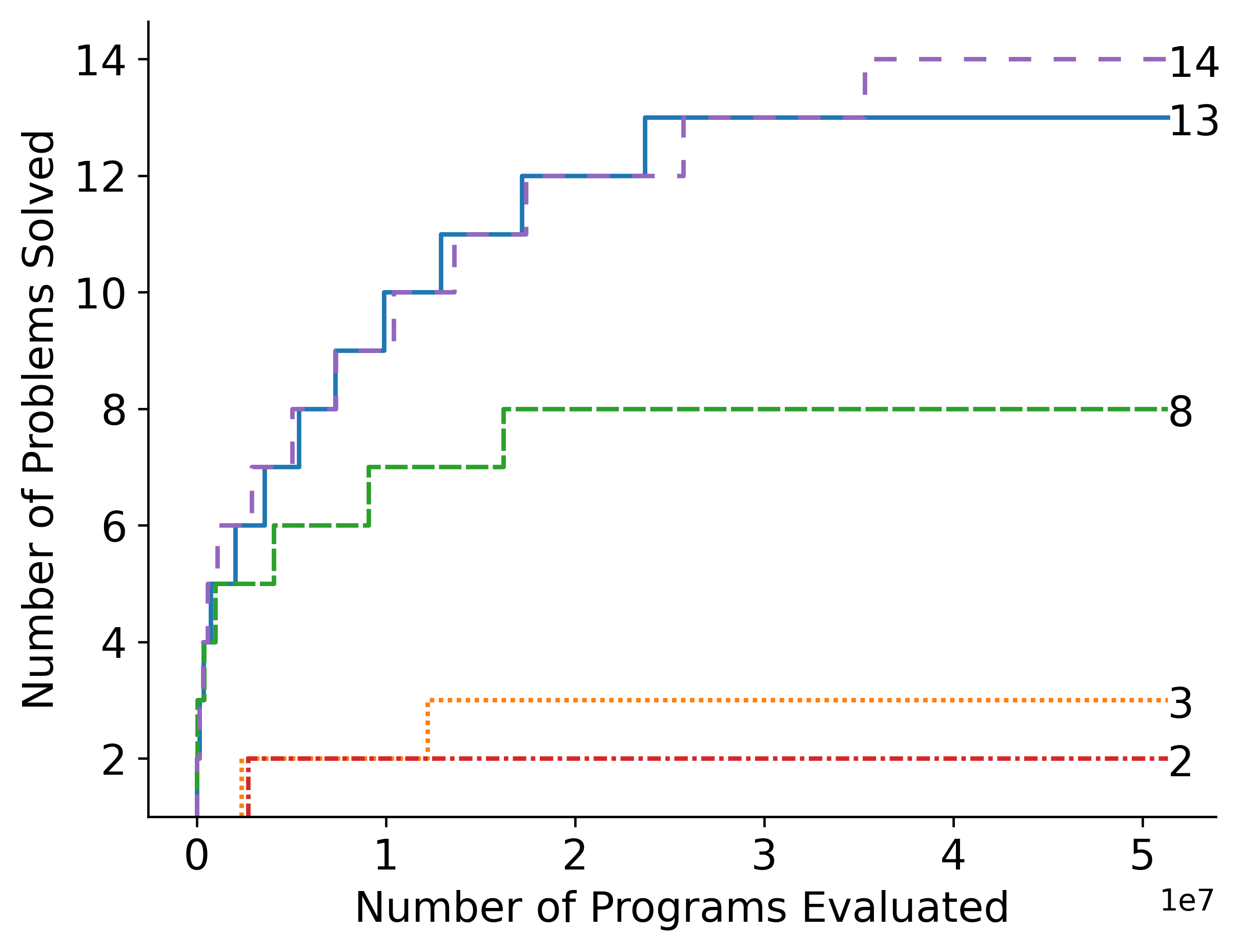}
\end{subfigure}  \hfil
\begin{subfigure}[b]{0.495\textwidth}
\includegraphics[width=\linewidth]{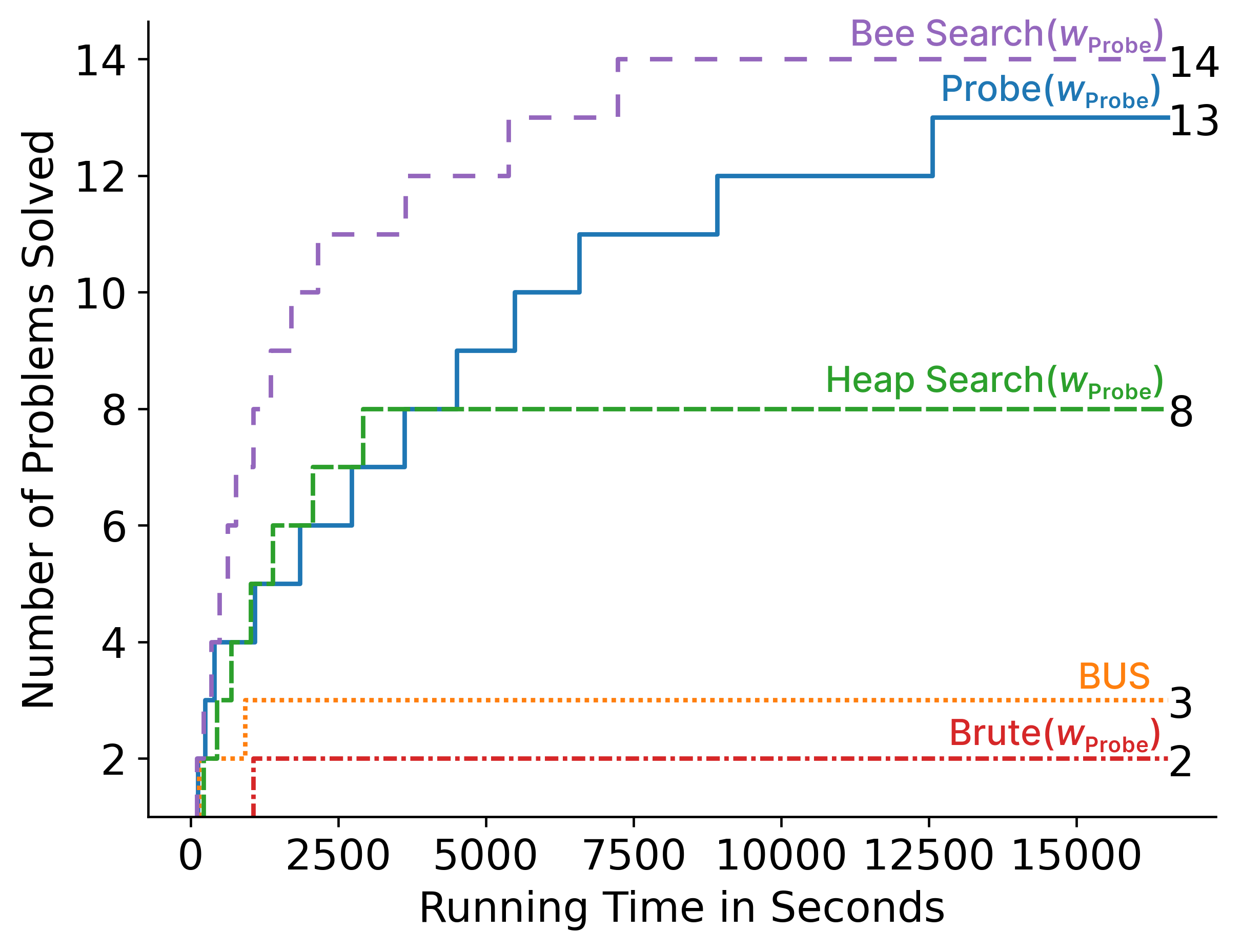}
\end{subfigure}

\caption{Number of problems solved per number of evaluations and running time for bit-vector domain. The two plots at the top show the results for the smaller DSL; the ones at the bottom show the results for the larger DSL.}
\label{fig:bit-vector-results}
\end{figure*}




Our discussion is divided by our key findings. 

\subsection*{\bee\ is not worse and often superior to others for a given cost function}

For a given cost function, \bee\ never performs worse in terms of the number of tasks solved, and often performs better than the other algorithms. For the SyGuS benchmark (Figure~\ref{fig:baseline-dsl-strings-results}), \bee\ with $w_{\bustle}$ solves 182 and 123 tasks for the smaller and larger DSL, respectively, while \bustle\ solves 180 and 117. \brute\ solves only 143 and 95 tasks with the same function. Similarly, \bee\ with $w_{\probe}$ is never worse than the other algorithms using $w_{\probe}$: it solves the same number of problems \probe{} solves for the smaller DSL and outperforms all algorithms in the larger DSL. We observe similar results in the $38$ tasks (Figure~\ref{fig:38-strings-results}) and in the bit-vector tasks (Figure~\ref{fig:bit-vector-results}). 

\subsection*{\bee\ with either $w_{\textsc{u}}$ or $w_{\probe}$ performs best in the evaluated domains} 

For the SyGuS benchmark (Figure~\ref{fig:baseline-dsl-strings-results}), \bee\ with $w_{\textsc{u}}$ solves more tasks than any other algorithm: 184 with the smaller DSL and 132 with the larger DSL. The second best algorithm in this domain is \bustle, which solves 180 tasks with the smaller DSL and 117 tasks with the larger DSL. Although the difference in terms of the number of tasks solved may seem small, it is substantial given that the tasks \bee\ solves and \bustle\ fails to solve are hard. 
For the $38$ handcrafted string tasks (Figure~\ref{fig:38-strings-results}), \probe{} and \bee\ with $w_{\probe}$ solve the largest number of tasks with the smaller DSL (29 tasks), while \bee\ with $w_{\bustle}$ and $w_{\textsc{u}}$ solve 28 tasks. \bustle\ comes next with 27 tasks solved. For the larger DSL, \bee\ with $w_\textsc{U}$ solves the largest number of tasks, 15, and is followed by \bee\ with $w_{\bustle}$ with 13; \bustle\ solves 12.6. A similar pattern is observed in the bit-vector domain (Figure~\ref{fig:bit-vector-results}), where \probe\ and \bee\ with $w_{\probe}$ solve 21 tasks with the smaller DSL. For the larger DSL, \bee\ with $w_{\probe}$ solves more problems than all other methods. The second best performing algorithm in bit-vector is \probe, with 13 problems solved, while \heapsearch{} is the third best algorithm with 8 problems solved.

\subsection*{$w_{\textsc{u}}$ performs better than $w_{\textsc{Bustle}}$}

\textsc{Bee Search} with $w_{\textsc{u}}$ outperforms all systems tested with $w_{\textsc{Bustle}}$, the other neural-based function, in the two evaluated domains: SyGuS benchmark and 38 string problems.

\subsection*{\bee\ overcomes the weaknesses of previous algorithms based on \bus}

These results suggest that \bee's best-first scheme is able to better use the information cost functions provide to the search than the ``truncation-based'' algorithms \probe\ and \bustle. The results also suggest that \bee's scheme of searching in the cost-tuple space is effective as it outperforms \brute{}  and \heapsearch{} by a large margin in all domains. \brute{} suffers from the fact that it generates a large number of programs that are never evaluated in search, which increases the algorithm's memory and time requirements. \bee\ does not suffer from this problem because its best-first search is with respect to program generation. \bee\ generates cost-tuple states that are never expanded, but the number of such states is much smaller than the number of programs \brute{} generates and are not evaluated. This is because the cost-tuple space is an abstraction of the original program space, where many programs are mapped to the same cost-tuple space. We conjecture that \heapsearch{} performs poorly in our experiments because it is unable to perform observational equivalence. In the SyGuS benchmark, even \bus substantially outperforms both \brute{} and \heapsearch{}. 

\subsection{Limitations of Evaluation}

In the context of ILP, \brute{} uses Answer Set Programming constraints to reduce the branching factor of search. The \brute{} version we evaluated in this paper is only an approximation of the original algorithm, as it is not clear how to adapt to Inductive Program Synthesis all the search enhancements developed in the context of ILP. Similarly, \heapsearch{} was originally evaluated in the context of parallel programming. In this paper, we evaluated only the sequential version of all algorithms. 

In some of our experiments, we observed that \bee\ performs as well as truncation-based algorithms (e.g., \probe\ in the bit-vector domain with the smaller DSL). We conjecture that these results can be explained by the nature of the cost function. For example, if the cost function does not provide helpful information for guiding the search, then both the exact and the truncated cost values will not be helpful for guiding the search. As another example, if a cost function is coarse-grained (e.g., all cost values are integers), then \bustle, \probe, and \bee will receive the same search signal to guide the search.  

\section{Related Work}\label{section:related-work}

The synthesis of programs has been studied for many years, starting with~\shortciteA{1969:PROW,1976:Smith}, and \shortciteA{Summers77}. It was used to solve tasks such as synthesis of database transactions~\shortcite{1990:DBTransactions,1993:DBTransactionsSynthesis}, system verification~\shortcite{1989:verification0,1990:verification1}, logic programming~\shortcite{1990:LogicProgrammingAndSynthesis,1994:LogicProgramSynthesis}, manipulation of bits~\shortcite{2005LezamaSketching,2009:gulwaniBitVectors}, strings~\shortcite{gulwani2011automating}, numbers~\shortcite{2012:NumberTransfromation}, synthesis of fault-tolerant circuits~\shortcite{2016:fsaCircuits}, and programmatic policies~\shortcite{VermaMSKC18,BastaniPS18,Marino2021}. 

Similarly to the application domains, there is also a large diversity of strategies for solving synthesis tasks. In constraint satisfaction algorithms, one transforms the synthesis task into a constraint satisfaction problem that can be solved with off-the-shelf SAT solvers~\cite{SolarLezama2009TheSA}. Stochastic search algorithms such as Simulated Annealing~\cite{programsynthesis_sa} and genetic algorithms~\cite{koza:book92} have also been applied to solve synthesis tasks. Stochastic search algorithms start with a candidate solution and use mutation operators to change that candidate into other candidates that might be closer to a solution. Enumerative algorithms systematically evaluate programs in the space defined by the DSL. 
We focus on enumerative algorithms since \textsc{Bee Search} is an enumerative method.

\subsection{Enumerative Methods}

Enumeration-based search has proven to be an effective approach and is used in many synthesizers ~\shortcite{bustle,probe,Woosuk2018,AlbarghouthiGK13,Udupa:2013}, including winners of SyGuS competitions~\shortcite{2016:syguswinner,Alur2017ScalingEP}. 
Enumerative methods can be classified into two categories: bottom-up and top-down. Bottom-up search (\bus{}) algorithms start with the shortest possible programs and use the rules of the symbolic language to generate longer programs by combining the shorter ones. BUS is an attractive search strategy because the programs generated are complete and thus can be executed, allowing observational equivalence checks to be performed~\cite{AlbarghouthiGK13,Udupa:2013,Woosuk2018,bustle,probe}. %
Top-down search algorithms start with a high-level structure of the program and enumerate the low-level structures. Top-down enumeration can only utilize weaker forms of equivalence~\cite{Woosuk2018,topdownEquivalence2017} since most of the programs generated in the search are incomplete and cannot be executed.  

\subsection{Guided Enumerative Search}

In guided enumerative search, instead of enumerating programs according to their AST size, the algorithms prioritize programs according to a function. One of the first guided search methods for program synthesis, \textsc{DeepCoder}~\cite{BalogGBNT16}, uses a top-down search. It uses a learned model to define a probability distribution over symbols in the language. Then, it performs a depth-first search that first explores the branches with a higher probability according to the model. \textsc{Euphony} also uses a probability distribution over production symbols of the underlying context-free grammar defining the programming language to guide a top-down search~\shortcite{Woosuk2018}, however, \textsc{Euphony}'s model considers the context in which a production rule is to be applied using the idea of probabilistic higher-order grammar. 
While different variations of using a learned model to guide top-down search algorithms have been introduced in the past ~\cite{chen2018executionguided,nips2018longprogramssynthesis,bunel2018leveraging,Devlin2017RobustFillNP,topdownEquivalence2017}, empirical evidence shows that they fail to outperform guided bottom-up search techniques~\cite{probe}. 

\textsc{TF-Coder}~\cite{tfcoder:2020} was the first system to utilize a function to guide a \bus\ algorithm. \textsc{TF-Coder} requires one to manually assign weight values to the operations based on their usage and complexity. During the search, \textsc{TF-Coder} prefers to combine programs with lower weights than programs with larger weights, thus biasing the search; the weight of a program is defined as the sum of the weights of the production rules used to generate the program. Since \textsc{TF-Coder} requires one to manually set weights for each operation, we did not consider it in our experiments. Similarly to \bustle\ and \probe, \textsc{TF-Coder} also suffers from loss of information, since it considers only integer cost values. 

\section{Conclusions}\label{section:conclusions}

In this paper, we showed that some of the current guided \bus\ algorithms suffer from a common problem: they can lose useful information given by the cost function because they only consider integer-valued costs. As a result, these algorithms do not perform best-first search with respect to the cost function used in the search. \heapsearch{} is a best-first guided bottom-up search algorithm that provably does not lose information from the cost function. However, \heapsearch{} sacrifices a key feature of \bus\ algorithms, which is the ability to eliminate observational equivalent programs. We presented an algorithm loosely inspired by the system \brute, from the ILP literature, to program synthesis, which we also referred to as \brute. \brute{} is able to perform search in best-first order and eliminate observational equivalent programs. However, \brute's search is best-first with respect to the evaluation of programs. As a result, many programs are generated but never evaluated, as they are more expensive than the solution program.   
We introduced \bee, a novel guided \bus\ algorithm that is guaranteed to perform search in best-first order when employing additive pre-generation cost functions and penalizing additive post-generation functions. In addition to performing search in best-first order, \bee\ is able to eliminate observational equivalent programs and its best-first search is with respect to the generation of programs. That is, \bee\ does not generate programs that are more expensive than the solution program. We also introduced a cost function that uses the neural model of \bustle. The difference between our function and \bustle's is that the former does not bound the penalty applied in the post-generation evaluation, as \bustle's cost function does. Empirical results on string manipulation and bit-vector problems showed that \bee\ was never worse than \probe\ and \bustle\ and can substantially outperform them, especially in larger program spaces. \bee\ outperformed \heapsearch{} and \brute\ by a large margin in both domains. Empirical results also showed that \bee\ with our cost function was the best performing system in the string manipulation domain. 

\section*{Acknowledgements}

We thank Thirupathi Reddy Emireddy for implementing the \bustle\ algorithm used in our experiments and the anonymous reviewers for their constructive feedback. This research was supported by Canada's NSERC and the CIFAR AI Chairs program. This research was enabled in part by support provided by the Digital Research Alliance of Canada.

\appendix

\section{Domain Specific Languages (DSLs)} \label{appendix:DSL}

\subsection{String Domain DSL}
\begin{figure}[H]
    $Start \rightarrow S \;|\; I \;|\; B $ \\
    $ S \rightarrow {\tt replace}(S,S,S) \;|\; {\tt concat}(S,S) \;|\; {\tt substr}(S, I, I)$\\
    \hspace*{12pt} $\; |\;{\tt ite}(B,S,S) \;|\; {\tt intToStr}(I) \;|\; {\tt charAt}(S, I)$\\
    \hspace*{12pt} $\; |\;{\tt toLower}(S) \;|\; {\tt toUpper}(S) \;|\; {\tt arg}0 \;|\; {\tt arg}1 \;|\; \ldots$\\
    \hspace*{12pt} $\; |\;{\tt lit-}0 \;|\; {\tt lit-}1 \;|\; \ldots$\\
    $ I \rightarrow {\tt strToInt}(S) \;|\; {\tt add}(I, I) \;|\; {\tt sub}(I, I) \;|\; {\tt mul}(I, I) \\$
    \hspace*{12pt} $\; |\;{\tt mod}(I, I) \;|\; {\tt length}(S) \;|\; {\tt indexOf}(S,S,I)$\\
    \hspace*{12pt} $\; |\;{\tt ite}(B, I, I) \;|\; {\tt find}(S,S) \;|\; {\tt arg}0 \;|\; {\tt arg}1 \;|\; \ldots$\\
    \hspace*{12pt} $\; |\;{\tt lit-}0 \;|\; {\tt lit-}1 \;|\; \ldots$\\
    $ B \rightarrow {\tt true} \;|\; {\tt false} \;|\; {\tt isEqual}(I,I) \;|\; {\tt isLess}(I,I)$\\
    \hspace*{12pt} $\; |\;{\tt isGreater}(I,I) \;|\; {\tt contains}(S,S)$\\
    \hspace*{12pt} $\; |\;{\tt isSuffixOf}(S,S) \;|\; {\tt isPrefixOf}(S,S)$\\
    \caption{Baseline DSL for string domain used in this paper}
    \label{fig:dsl-string}
\end{figure}

\subsection{BitVector Domain DSL}

\begin{figure}[H]
    $Start \rightarrow BV \;|\; B $ \\
    $ BV \rightarrow {\tt xor}(BV, BV) \;|\; {\tt and}(BV, BV) \;|\; {\tt or}(BV, BV)$\\
    \hspace*{12pt} $\; |\;{\tt neg}(BV) \;|\; {\tt not}(BV) \;|\; {\tt add}(BV, BV) \;|\; {\tt mul}(BV, BV)$\\
    \hspace*{12pt} $\; |\;{\tt udiv}(BV, BV) \;|\; {\tt urem}(BV, BV) \;|\; {\tt lshr}(BV, BV) $\\
    \hspace*{12pt} $\; |\;{\tt ashr}(BV, BV) \;|\; {\tt shl}(BV, BV) \;|\; {\tt sdiv}(BV, BV)$\\
    \hspace*{12pt} $\; |\;{\tt srem}(BV, BV) \;|\; {\tt sub}(BV, BV) \;|\; {\tt ite}(B, BV, BV)$\\
    \hspace*{12pt} $\; |\;{\tt arg}0 \;|\; {\tt arg}1 \;|\; \ldots \;|\; {\tt lit-}0 \;|\; {\tt lit-}1 \;|\; \ldots $\\
    $ B \rightarrow {\tt true} \;|\; {\tt false} \;|\; {\tt isEqual}(BV,BV) \;|\; {\tt ult}(BV,BV)$\\
    \hspace*{12pt} $\; |\;{\tt ule}(BV,BV) \;|\; {\tt slt}(BV,BV) \;|\; {\tt sle}(BV,BV)$\\
    \hspace*{12pt} $\; |\;{\tt ugt}(BV,BV) \;|\; {\tt redor}(BV) \;|\; {\tt and}(BV,BV) $\\
    \hspace*{12pt} $\; |\;{\tt or}(BV,BV) \;|\; {\tt not}(BV) \;|\; {\tt uge}(BV, BV)$\\
    \hspace*{12pt} $\; |\;{\tt sge}(BV,BV) \;|\; {\tt sgt}(BV, BV)$\\
    \caption{Baseline DSL for bit-vector domain used in this paper}
    \label{fig:dsl-bitvector} 
\end{figure}

\vskip 0.2in
\bibliography{jair}
\bibliographystyle{theapa}

\end{document}